\documentclass[10pt,letterpaper]{article}
\usepackage[margin=1in]{geometry}
\usepackage{times}
\usepackage{natbib}
\setcitestyle{authoryear,round,semicolon}

\usepackage{amsmath,amsfonts,bm}

\def\eqref#1{equation~\ref{#1}}

\def\1{\bm{1}}

\DeclareMathAlphabet{\mathsfit}{\encodingdefault}{\sfdefault}{m}{sl}
\SetMathAlphabet{\mathsfit}{bold}{\encodingdefault}{\sfdefault}{bx}{n}

\usepackage[utf8]{inputenc}
\usepackage[T1]{fontenc}
\usepackage[colorlinks=true,citecolor=blue,linkcolor=black,urlcolor=blue]{hyperref}
\usepackage{url}
\usepackage{booktabs}
\usepackage{amsfonts}
\usepackage{amsmath}
\usepackage{amssymb}
\usepackage{amsthm}
\usepackage[ruled,linesnumbered]{algorithm2e}
\usepackage{graphicx}
\usepackage{multirow}
\usepackage{subcaption}
\usepackage{makecell}
\usepackage{color}
\usepackage{xcolor}
\usepackage{enumitem}
\usepackage{xspace}

\makeatletter
\renewenvironment{abstract}{%
  \if@twocolumn
    \section*{\abstractname}%
  \else
    \small
    \begin{center}%
      {\bfseries \abstractname\vspace{-.5em}\vspace{\z@}}%
    \end{center}%
    \quote
  \fi}
  {\if@twocolumn\else\endquote\fi}
\makeatother

\newcommand{\safeincludegraphics}[2][]{%
  \IfFileExists{#2}{\includegraphics[#1]{#2}}{%
    \fbox{\parbox[c][0.23\textheight][c]{0.92\linewidth}{\centering
      \textbf{Figure file not supplied}\\[0.5em]
      \texttt{\detokenize{#2}}}}%
  }%
}

\newtheorem{theorem}{Theorem}
\newtheorem{lemma}{Lemma}
\newtheorem{definition}{Definition}

\newcommand{\cs}[1]{{\color{black} #1}}
\newcommand{\ALG}{{PC-SubMax\xspace}}

\title{\ALG: Efficient Prompt Compression via Regularized Submodular Maximization}

\author{
Ziyi Zhang$^{1}$ \quad
Shuang Cui$^{1}$ \quad
Haotian Zhang$^{1}$ \quad
Xiaoyu Wang$^{1}$\\[0.45em]
\small $^{1}$Soochow University (\texttt{zyzhang47@stu.suda.edu.cn})
}
\date{}

\begin{document}

\maketitle

\begin{abstract}
While large language models (LLMs) are increasingly deployed in long-context scenarios, \cs{lengthy prompts can increase inference costs and latency and exacerbate the ``lost-in-the-middle'' phenomenon.} Selective prompt compression offers \cs{a model-agnostic approach to alleviating these issues. However, methods based on fixed token- or sentence-level importance scores may overlook how content contributions change with the selected subset, limiting their ability to account for inter-sentence redundancy. Compression procedures that rely on autoregressive LLM scoring can also introduce substantial overhead.}
\cs{We propose PC-SubMax, a theoretically grounded framework that formulates selective prompt compression as regularized monotone submodular maximization under a knapsack constraint. The objective is $U(S)-\ell(S)$, where the monotone submodular utility $U$ combines information coverage, query relevance, and log-determinant diversity, and the non-negative modular penalty $\ell$ captures token cost. Through diminishing marginal returns, the objective evaluates each sentence's contribution relative to the selected content. To optimize this objective, we develop the Regularized Greedy+Max (RGM) algorithm, which deterministically returns a feasible set $Q$ satisfying $U(Q)-\ell(Q)\geq \frac{1}{2}U(O)-\ell(O)$, where $O$ is an optimal feasible solution to the regularized problem. RGM uses $O(n\kappa)$ value-oracle queries, where $n$ is the number of candidate sentences and $\kappa$ is the maximum feasible subset size. PC-SubMax uses encoder representations and avoids autoregressive LLM scoring during compression. Experiments across seven diverse benchmarks demonstrate competitive downstream performance with low compression overhead.}
\end{abstract}

\section{Introduction}\label{sec:intro}

Large language models (LLMs) have demonstrated exceptional capabilities across a wide array of natural language processing tasks~\citep{ouyang2022training,touvron2023llama}. To tackle complex real-world problems, LLMs are increasingly deployed in long-context scenarios, such as in-context learning (ICL)~\citep{brown2020language}, retrieval-augmented generation (RAG)~\citep{lewis2020retrieval}, and multi-document reasoning~\citep{lin2026docsage}. However, providing LLMs with excessively long prompts presents three \cs{practical} challenges. First, \cs{processing long inputs increases inference cost and latency; in particular, standard dense self-attention incurs quadratic computational cost in the input length during prompt processing}~\citep{vaswani2017attention}. Second, \cs{irrelevant information in lengthy prompts can distract the model and impair reasoning performance}~\citep{shi2023large}. Third, LLMs \cs{can exhibit} the well-documented ``lost-in-the-middle'' phenomenon~\citep{liu2024lost}, where \cs{performance deteriorates when relevant information is located in the middle rather than near the beginning or end of the context}.

\subsection{Related Work} To alleviate these issues, prompt compression has emerged as \cs{a practical approach to reducing input length while retaining useful information}. Existing methods can be broadly categorized based on \cs{their access requirements for the target model} into \textit{white-box} and \textit{black-box} approaches. White-box methods \cs{depend on access to the target model's internal representations or require modifications to its parameters or computation, as in methods that encode context into model-specific soft prompts. Such methods generally cannot be deployed directly through standard text-only APIs. Black-box methods avoid these requirements and produce compressed text that can be supplied to the target model through its ordinary input interface. Two common paradigms for text-based compression are} \textit{generative and selective} \cs{compression}. Generative methods synthesize new text to represent the original content. However, \cs{rewriting can alter or omit fine-grained factual details, including exact numbers needed for mathematical reasoning}. \cs{Selective compression instead retains selected tokens, phrases, or sentences from the original text, preserving source wording while remaining compatible with standard text interfaces}.

\cs{Selective compression methods can be broadly categorized into
\textit{task-agnostic} and \textit{task-aware} approaches.
Task-agnostic methods estimate content importance without conditioning
compression on a specific downstream query. Selective-Context
~\citep{li2023compressing} uses self-information to score lexical units,
while LLMLingua~\citep{jiang2023llmlingua} combines language-model-based
importance estimates with budget allocation and iterative compression.
LLMLingua-2~\citep{pan2024llmlingua2} instead uses a contextual token
classifier trained on distilled retention labels, reducing online
scoring cost. Other methods incorporate richer selection mechanisms:
PartPrompt~\citep{mao2025parse} uses hierarchical linguistic structure
and budgeted tree pruning, while LLM-DPC~\citep{hu2026dynamic} learns
a sequential token-removal policy through reinforcement learning.} \cs{Task-aware methods additionally incorporate query or domain
information. LongLLMLingua~\citep{jiang2024longllmlingua} uses
contrastive perplexity to prioritize question-relevant content.
CPC~\citep{liskavets2025prompt} performs sentence-level compression
by ranking query relevance using a contrastively trained
context-aware encoder. Domain-specific methods such as
ShortenDoc~\citep{yang2026less} and
SR-C3~\citep{atandoh2026less} tailor compression to code generation
and few-shot sentiment analysis, respectively.}

\cs{These methods improve compression through more informative
scores, structural constraints, or learned selection policies.
However, identifying important content is only part of budgeted
compression: the retained units must also form a useful subset, and selecting that subset must also be computationally economical.
Existing selection strategies have three limitations relevant to this problem.}

\begin{enumerate}[leftmargin=*]
    \item \cs{\textit{Incomplete modeling of subset utility.}
    Methods that select content using precomputed retention or
    relevance scores, such as the token-ranking strategy of
    LLMLingua-2~\citep{pan2024llmlingua2} and sentence ranking in
    CPC~\citep{liskavets2025prompt}, can incorporate information
    from the original context. However, their ranking criteria
    do not explicitly evaluate how much additional information
    a candidate contributes after other content has been retained.
    Several highly ranked units may convey overlapping evidence,
    consuming budget that could otherwise preserve complementary
    information. This limits the ability of fixed-score selection
    to balance coverage and redundancy across the selected subset.}

    \item \cs{\textit{Limited optimization guarantees for semantic
    subset selection.}
    Retention prediction and relevance ranking provide practical
    selection rules, but do not by themselves characterize
    solution quality for an explicit objective that jointly models
    information coverage and redundancy. Once these interactions
    are represented at the set level, selection under a token
    budget becomes a combinatorial optimization problem.
    A corresponding approximation analysis is needed to quantify
    the attained objective value relative to an optimal feasible
    subset.}

    \item \cs{\textit{Compression-stage scoring overhead.}
    Language-model-based compressors, such as LLMLingua~\citep{jiang2023llmlingua} and
    LongLLMLingua~\citep{jiang2024longllmlingua}, require additional
    causal-language-model evaluations to score the input during
    compression. These evaluations can introduce substantial
    online overhead for long prompts, reducing the computational
    savings obtained from shorter downstream inputs.
    The cost of the compression procedure therefore remains
    an important consideration alongside the quality of its output.}
\end{enumerate}

\cs{To address these limitations, we propose} \textbf{PC-SubMax}
(\textbf{P}rompt \textbf{C}ompression via \textbf{Sub}modular
\textbf{Max}imization),
\cs{a theoretically grounded framework for sentence-level prompt
compression under a token budget. PC-SubMax models information coverage, semantic diversity,
and query relevance through a normalized monotone submodular
utility. Its coverage and diversity terms assign diminishing
marginal value to information already represented in the selected
set, explicitly accounting for overlap during selection.
A modular token-cost penalty further balances retained utility
against token expenditure within the hard budget, yielding
regularized monotone submodular maximization under a knapsack
constraint. To optimize this objective, we develop the Regularized Greedy+Max (RGM) algorithm with a provable regularized
approximation guarantee. PC-SubMax computes its objective from
pretrained encoder representations and evaluates marginal gains
incrementally, supporting subset-dependent selection without
causal-language-model scoring during compression.
Our key contributions are summarized as follows:}

\begin{itemize}[leftmargin=2em]
    \item \cs{\textbf{Regularized modeling of sentence selection.}
    We formulate selective prompt compression as maximizing $U(S)-\ell(S)$ under a token budget, where $U$ is a normalized monotone submodular utility and $\ell$ is a non-negative modular token penalty. PC-SubMax combines information coverage, log-determinant diversity, and query relevance, with single-hop and multi-hop relevance options. The coverage and diversity terms capture subset-dependent utility, while the penalty provides explicit control over token expenditure within the feasible budget.}

    \item \cs{\textbf{An improved deterministic approximation guarantee.}
    We develop Regularized Greedy+Max (RGM), which uses cost-scaled marginal-density selection and one-item augmentation at every greedy prefix. For any normalized monotone submodular utility $U$, non-negative modular penalty $\ell$, and positive knapsack weights, RGM deterministically returns a feasible set $Q$ satisfying $U(Q)-\ell(Q)\geq \frac{1}{2}U(O)-\ell(O)$ for every feasible comparator $O$, using $O(n\kappa)$ value-oracle queries. Here $n$ is the ground-set size and $\kappa$ is the maximum feasible cardinality. In this monotone-utility setting, RGM improves the $(1/4,1)$ guarantee of PartEnum-Twin-Greedy to $(1/2,1)$ and reduces its $O(n^4)$ query complexity to $O(n\kappa)$~\citep{gong2024budget}.}

    \item \cs{\textbf{Efficient implementation and accelerated selection.}
    PC-SubMax uses pretrained sentence representations and incremental marginal evaluation, requiring neither compression-specific model training nor target-LLM scoring during compression. We further develop a lazy variant of RGM that achieves a $(1/2-\varepsilon,1)$ guarantee with
    $O\!\left((n/\varepsilon)\left[1+\log\!\left(B/(\varepsilon c_{\min})\right)\right]\right)$
    value-oracle queries, where $0<\varepsilon<1/2$, $B$ is the budget, and $c_{\min}$ is the minimum individually feasible candidate cost. \cs{The accelerated algorithm and its analysis, together with a detailed breakdown of preprocessing and marginal-evaluation costs, are provided in Appendices~\ref{sec:fast_rgm}--\ref{sec:appendix_complexity}.}}

    \item \cs{\textbf{Evaluation across diverse downstream tasks.}
    Experiments on seven benchmarks spanning summarization, mathematical reasoning, multi-hop question answering, and long-context understanding demonstrate competitive downstream performance against five representative prompt-compression baselines. PC-SubMax achieves the lowest measured compression time on four of the six benchmarks used for timing evaluation. The compressed prompts consist of original sentences and can be supplied directly to target LLMs through standard text interfaces.}
\end{itemize}

\cs{Complete proofs of our lemmas and theorems are provided in} Appendixs~\ref{sec:appendix_proofs}--\ref{sec:appendix_complexity}.

\section{Problem Formulation}\label{sc:problem}
We begin by introducing the formal definitions of submodularity and the prompt compression problem, and then cast selective prompt compression as regularized \cs{monotone} submodular maximization subject to a knapsack constraint.

\begin{definition}[Submodular Function]
\label{def:submodular}
A set function $f: 2^{\mathcal{V}} \to \mathbb{R}$ defined on a \cs{finite} ground set $\mathcal{V}$ is \textit{submodular} if and only if for all $A \subseteq B \subseteq \mathcal{V}$ and \cs{$e \in \mathcal{V}\setminus B$}: $f(A \cup \{e\}) - f(A) \geq f(B \cup \{e\}) - f(B)$.
\end{definition}

\cs{
We write $f(e\mid S)=f(S\cup\{e\})-f(S)$ for the marginal gain of adding $e$ to $S$. A set function is \textit{normalized} if $f(\emptyset)=0$ and \textit{monotone} if $f(A)\leq f(B)$ whenever $A\subseteq B$. The modular penalties considered here are additive functions of the form $\ell(S)=\sum_{e\in S}\ell_e$, with $\ell_e\geq 0$.
}

\begin{definition}[\cs{Sentence-Level Selective Prompt Compression}]
\label{def:compression}
\cs{
Given a prompt $\mathcal{D}=(s_1,\ldots,s_n)$ and a token budget $B>0$, sentence-level selective prompt compression chooses a subset $S\subseteq\mathcal{V}=\{s_1,\ldots,s_n\}$ and concatenates the selected sentences in their original order to form a compressed prompt $\mathcal{D}'$. Sentence occurrences are distinguished by their positions, even when their text is identical. The goal is to retain information useful for downstream inference within the available budget.
}
\end{definition}

\cs{
Evaluating downstream task performance during subset selection would require repeated target-LLM calls and a task-specific evaluation signal. We therefore construct an efficiently evaluable surrogate from information coverage, semantic diversity, and query relevance. For each compression instance, the original prompt and any associated query are fixed, and the selected sentence subset is the optimization variable. We formalize the model as follows:
}

\begin{itemize}[leftmargin=2em]
    \item \textbf{Knapsack Constraint (Token Budget):}
    \cs{
    We adopt an additive token-cost model in which each sentence $s_i$ has a fixed cost $c_i=c(s_i)>0$. The budget $B$ is allocated to the selectable context after accounting for any fixed prompt components. A selected subset is feasible if
    $c(S)=\sum_{s_i\in S}c_i\leq B$.
    }

    \item \textbf{Regularized Submodular Surrogate:}
    We design a \cs{normalized,} non-negative monotone submodular semantic utility $U(S)$ and subtract a non-negative modular token price $\ell(S)$.
    \cs{
    The resulting objective $G(S)=U(S)-\ell(S)$ balances semantic utility against token expenditure. The hard budget determines admissible subsets, while the penalty controls the cost of using tokens within that budget. Our coverage and diversity terms represent information overlap through diminishing marginal returns. The specific constructions of $U$ and $\ell$ are given in Section~\ref{sc:method}.
    }
\end{itemize}

\cs{
The resulting optimization problem is regularized monotone submodular maximization under a knapsack constraint:
}
\begin{equation}
\cs{
\begin{aligned}
    \max_{S\subseteq\mathcal{V}}\quad
        & G(S)=U(S)-\ell(S)\qquad
    \text{s.t.}\quad
        & c(S)\leq B.
\end{aligned}
}
\label{eq:knapsack}
\end{equation}

\cs{
Here monotonicity refers to $U$. Subtracting the modular penalty preserves submodularity, but the net objective $G$ can be non-monotone. The general formulation permits the knapsack weights $c_i$ and penalty coefficients $\ell_i$ to vary independently; PC-SubMax instantiates the penalty as $\ell_i=\lambda_{\mathrm{tok}}c_i/B$, where $\lambda_{\mathrm{tok}}\geq 0$.
} \cs{
This problem class is NP-hard, since it includes cardinality-constrained maximum coverage as a special case when $\ell\equiv 0$ and all costs are equal. We therefore seek efficient algorithms with regularized approximation guarantees. Let $O$ be an optimal feasible solution to~\eqref{eq:knapsack}. An $(\alpha,1)$ regularized approximation, with $\alpha\in[0,1]$, returns a feasible set $Q$ satisfying $ U(Q)-\ell(Q)\geq \alpha U(O)-\ell(O).$
} \cs{
This guarantee quantifies the utility retained by the algorithm while accounting for the optimal solution's full penalty. The following section presents our objective construction and an algorithm achieving $\alpha=1/2$.
}

\begin{figure}[!htbp]

\centering
\safeincludegraphics[width=\linewidth]{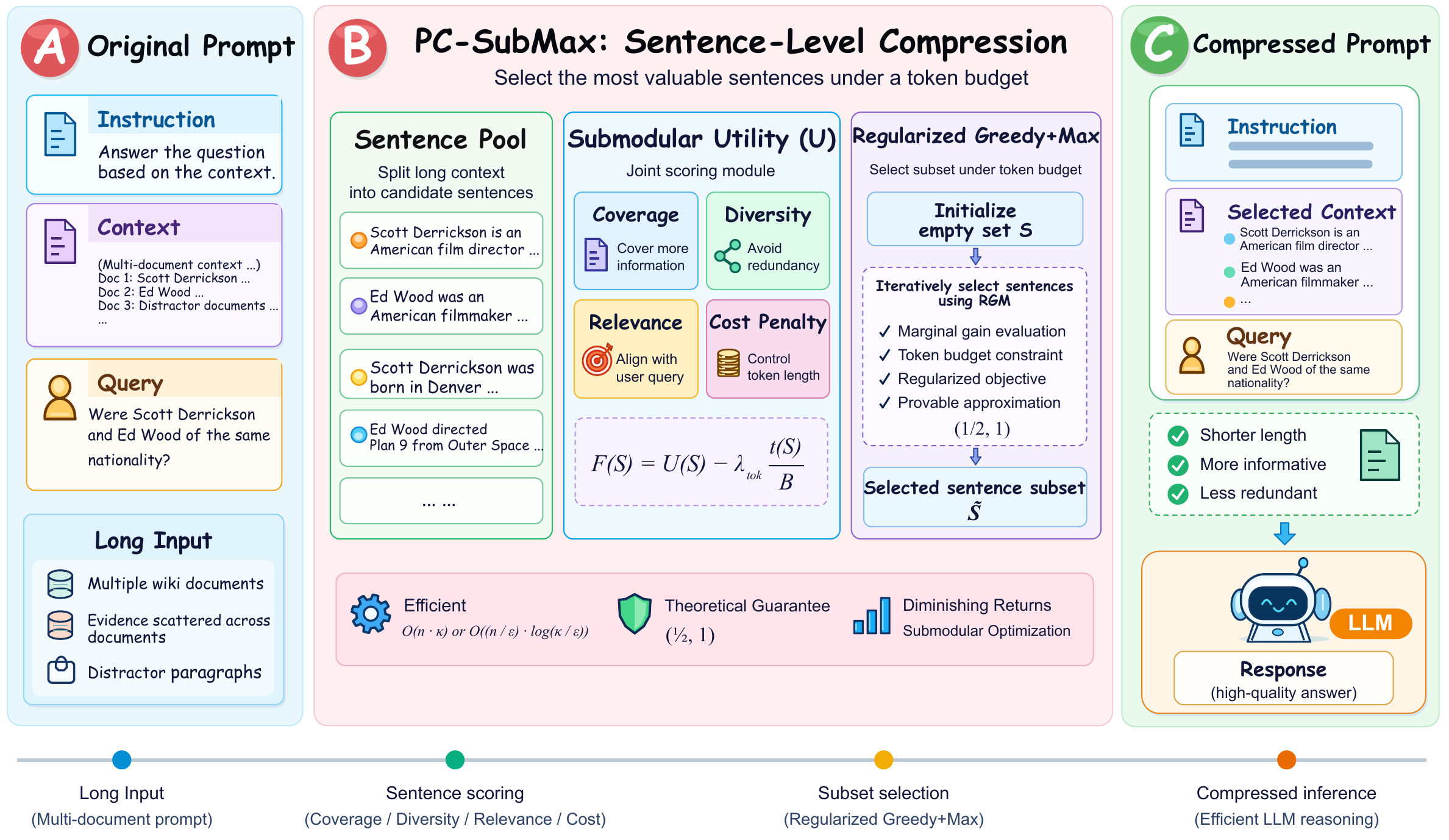}

\caption{Overview of PC-SubMax.}
\label{fig:overview}

\end{figure}

\section{PC-SubMax: Prompt Compression via \cs{Regularized} Submodular Maximization}\label{sc:method}

In this section, we present the overall framework of \textbf{PC-SubMax} (illustrated in Figure~\ref{fig:overview}). \cs{The framework has two stages: objective construction and algorithmic optimization. We first combine coverage, diversity, and query relevance into a normalized monotone submodular utility, and subtract a modular token penalty to balance utility against token usage. We then optimize the regularized objective under a knapsack constraint using deterministic Regularized Greedy+Max (RGM). The algorithm achieves a $(1/2,1)$ regularized guarantee with $O(n\kappa)$ value-oracle queries.}

\subsection{\cs{Regularized Submodular Objective Construction}}
We design the objective function in two parts. First, we define a semantic utility function $U(S)$ as a weighted combination of three components, each capturing a distinct semantic desideratum for prompt compression:
$U(S)
=
\lambda_{\text{cov}} \cdot f_{\text{cov}}(S)
+
\lambda_{\text{div}} \cdot f_{\text{div}}(S)
+
\lambda_{\text{rel}} \cdot f_{\text{rel}}(S),
\label{eq:semantic_utility}$
where $S \subseteq \mathcal{V}$ is the selected sentence subset, and $\lambda_{\text{cov}}, \lambda_{\text{div}}, \lambda_{\text{rel}} \geq 0$ control the importance of coverage, diversity, and relevance, respectively. \cs{These weights remain fixed during subset selection.}

\cs{Second, the knapsack constraint specifies a maximum admissible token cost. To price token usage within this feasible region, we assign sentence $s_i$ the penalty $\ell_i=\lambda_{\text{tok}}c(s_i)/B$ and define the modular penalty}
$\ell(S)=\sum_{s_i\in S}\ell_i
=\lambda_{\text{tok}}\frac{c(S)}{B}.
\label{eq:modular_cost}$
\cs{The regularized objective is}
$G(S)=U(S)-\ell(S),
\label{eq:objective}$
where $\lambda_{\text{tok}}\geq 0$ controls the strength of the token-cost regularization.

\paragraph{Coverage Function $f_{\text{cov}}(\cdot)$.}

\cs{
The coverage term measures lexical coverage of the original prompt. Let $W_i$ be the set of unique words in sentence $s_i$, and let $W=\bigcup_{i=1}^n W_i$. We define
$f_{\mathrm{cov}}(S)=\frac{|\bigcup_{s_i\in S}W_i|}{|W|}$,
with $f_{\mathrm{cov}}\equiv0$ if $W=\emptyset$. This term rewards vocabulary newly covered by the selected sentences. It provides a lexical coverage signal that complements embedding-based diversity and query relevance.
}

\begin{lemma}
\label{lem:cov}
The coverage function $f_{\text{cov}}(\cdot)$ is \cs{normalized,} non-negative, monotone, and submodular.
\end{lemma}

\paragraph{Diversity Function $f_{\text{div}}(\cdot)$.}

\cs{
To complement lexical coverage, we use a log-determinant diversity term inspired by determinantal subset selection~\citep{kulesza2012determinantal}. Let $z_i\in\mathbb{R}^d$ be the unit-normalized embedding of sentence $s_i$, and let $Z$ have rows $z_i^\top$. The cosine kernel $K=ZZ^\top$ is positive semidefinite. For $S\subseteq\mathcal V$, let $K_S$ denote its principal submatrix indexed by $S$, and let $I_{|S|}$ be the identity matrix of the same size. We define
$f_{\mathrm{div}}(S)=\log\det(I_{|S|}+K_S)$, where $f_{\mathrm{div}}(\emptyset)=0.$
Since $I_{|S|}+K_S\succeq I_{|S|}$, this quantity is finite and non-negative. It rewards diversity among the selected embedding directions: $m$ identical unit vectors yield $f_{\mathrm{div}}(S)=\log(1+m)$, whereas $m$ mutually orthogonal unit vectors yield $m\log 2$ when $d\geq m$. Thus overlapping representations receive diminishing incremental reward.
}

\begin{lemma}
\label{lem:div}
The diversity function $f_{\text{div}}(\cdot)$ is \cs{normalized,} non-negative, monotone, and submodular.
\end{lemma}

\paragraph{Relevance Function $f_{\text{rel}}(\cdot)$.}

\cs{For tasks with a query $q$, relevance directs selection toward information useful for answering that query. This complements coverage and diversity, since varied or previously uncovered content may still be unrelated to the question. Irrelevant context can also distract LLMs during reasoning~\citep{shi2023large}. Let $z(x)$ be the unit-normalized representation of text $x$ obtained from the same encoder. We define the positive part of cosine similarity as $\rho(s_i,x)=[z_i^\top z(x)]_+$, where $[t]_+=\max\{t,0\}$. The single-hop relevance score is} $r_i^{(1)}=\rho(s_i,q).
\label{eq:single_hop_relevance}$
\cs{For multi-hop question answering, direct similarity to $q$ may underweight bridge evidence whose relevance becomes apparent after another sentence introduces an entity or relation. We therefore precompute a bridge score using beam search over evidence paths of at most $H\geq2$ hops. The bridge bonus reflects the relevance gained when preceding evidence is appended to the query, together with a title-matching condition when the input has the required structure. Let $b_i\in[0,1]$ be the highest score among the examined paths of length at least two ending at $s_i$, with $b_i=0$ if no such path is examined. The resulting score is} $r_i^{(\mathrm{MH})}=\min\left\{1,\;r_i^{(1)}+\frac{b_i}{H}\right\}.
\label{eq:multihop_relevance}$
The construction of $b_i$, including the conditional relevance gain, the title-bridge test, the path score, and the beam-search procedure, is deferred to Appendix~\ref{sec:appendix_multihop}. %

For single-hop tasks we use $r_i=r_i^{(1)}$, whereas for multi-hop tasks we use $r_i=r_i^{(\mathrm{MH})}$. In both cases, all scores are computed before subset selection and remain fixed throughout RGM. The relevance function is therefore $f_{\text{rel}}(S)=\sum_{\cs{s_i\in S}}r_i.
\label{eq:relevance_function}$

\begin{lemma}
\label{lem:rel}
Under either the single-hop or multi-hop scoring rule, the relevance function $f_{\text{rel}}(\cdot)$ is \cs{normalized,} non-negative, monotone, and modular.
\end{lemma}

\paragraph{Token-Cost Regularization $\ell(\cdot)$.}
\cs{The penalty has constant marginal cost $\ell(s_i\mid S)=\lambda_{\text{tok}}c_i/B$. It favors a shorter feasible subset when the additional utility of a longer one does not offset its additional token penalty. For feasible $S$, $0\leq\ell(S)\leq\lambda_{\text{tok}}$; changing $\lambda_{\text{tok}}$ changes this trade-off while leaving the feasible region unchanged.}
\begin{lemma}
\label{lem:bud}
The token-price function $\ell(\cdot)$ is normalized, non-negative, monotone, and modular.
\end{lemma}

Combining the results above, we establish the properties of the composite objective:

\begin{theorem}
\label{thm:objective}
The semantic utility $U(\cdot)$ is normalized, non-negative, monotone, and submodular. Consequently, the regularized objective $G(\cdot)=U(\cdot)-\ell(\cdot)$ is normalized and submodular, but it may be non-monotone and may take negative values.
\end{theorem}

\subsection{Sentence Selection Procedure}

Having established the decomposition $G=U-\ell$, we now turn to the algorithmic challenge of maximizing it under the knapsack constraint in~\eqref{eq:knapsack}. \cs{To address this problem, we propose \emph{Regularized Greedy+Max} (RGM), a deterministic algorithm presented in Algorithm~\ref{alg:pc_submax}.} \cs{The algorithm and analysis apply to arbitrary non-negative modular prices $\ell_i$, independently of the positive knapsack weights $c_i$. PC-SubMax uses the specialization $\ell_i=\lambda_{\text{tok}}c(s_i)/B$.}


\begingroup
\SetAlgoNoLine
\SetAlgoNoEnd
\DontPrintSemicolon
\SetNlSty{textnormal}{}{:}
\SetAlgoNlRelativeSize{0}
\SetAlgoCaptionSeparator{}
\IncMargin{1.5em}

\begin{algorithm}[!htbp]
\caption{Regularized Greedy+Max (RGM) with exact marginal evaluations}
\label{alg:pc_submax}
\KwIn{\cs{candidate set $\mathcal{V}$, costs $c_i>0$, budget $B>0$, monotone submodular utility $U$, and modular penalties $\ell_i\geq0$}}
\KwOut{\cs{feasible selected subset $Q$}}

$S\gets\emptyset$; $Q\gets\emptyset$\;
\While{\textnormal{true}}{
    $E(S)\gets\{e\in\mathcal{V}\setminus S:c(S)+c_e\leq B\}$\;
    \lIf{$E(S)=\emptyset$}{\textbf{break}}
    \For{$e\in E(S)$}{
        $\Delta_e\gets U(S\cup\{e\})-U(S)$\;
    }
    $b\gets\arg\max_{e\in E(S)}(\Delta_e-\ell_e)$\;
    \lIf{$G(S\cup\{b\})>G(Q)$}{$Q\gets S\cup\{b\}$}
    $a\gets\arg\max_{e\in E(S)}(\Delta_e-2\ell_e)/c_e$\;
    \lIf{$\Delta_a-2\ell_a\leq0$}{\textbf{break}}
    $S\gets S\cup\{a\}$\;
    \lIf{$G(S)>G(Q)$}{$Q\gets S$}
}
\Return{$Q$}\;
\end{algorithm}
\endgroup

\paragraph{Algorithm Overview.}
\cs{RGM maintains a greedy path $S$ and the best recorded feasible set $Q$. At each prefix, it evaluates every feasible utility marginal $\Delta_e=U(e\mid S)$ once and reuses it in two selection rules. The augmentation rule chooses $b$ maximizing $\Delta_e-\ell_e$ and records $S\cup\{b\}$ if it improves $Q$. It considers the unnormalized marginal gain in $G$, complementing the density-based path. The path rule chooses $a$ maximizing $(\Delta_e-2\ell_e)/c_e$ and adds it to $S$ only when $\Delta_a-2\ell_a>0$. The factor two is used to construct the path, while all candidate solutions are compared using $G=U-\ell$. Summing the accepted marginal gains shows that every greedy prefix satisfies $U(S)\geq2\ell(S)$ and hence $G(S)\geq U(S)/2$. This invariant connects the doubled penalty in the selection rule
to the regularized approximation analysis. With $U(S)$ cached, the value of an augmented set is computed as $G(S)+\Delta_e-\ell_e$. The algorithm compares the empty set, accepted prefixes, and the best feasible one-item augmentation of each scanned prefix. For the PC-SubMax penalty $\ell_e=\lambda_{\mathrm{tok}}c_e/B$, the path density simplifies to}
$\frac{U(e\mid S)-2\ell_e}{c_e}
=\frac{U(e\mid S)}{c_e}-\frac{2\lambda_{\mathrm{tok}}}{B}.
\label{eq:rgm_density}$
\cs{Thus, in this specialization, cost scaling preserves the utility-density ordering and sets the stopping threshold $U(e\mid S)/c_e>2\lambda_{\mathrm{tok}}/B$. The penalty also enters the augmentation rule and the comparison of candidate solutions. For general modular penalties, both the path ordering and its stopping condition depend on $\ell_e$. Ties are resolved by original sentence index in PC-SubMax, making selection deterministic. The empty incumbent ensures that the returned objective value is non-negative.}

\cs{
\begin{theorem}[Regularized RGM guarantee]
\label{thm:pc_guarantee}
Let $U:2^{\mathcal V}\to\mathbb R_{\geq0}$ be normalized, monotone,
and submodular.  Let $B>0$, let $c_e>0$ be arbitrary knapsack weights,
and let $\ell(S)=\sum_{e\in S}\ell_e$ be an arbitrary non-negative
modular cost.  Under exact marginal evaluations and exact maximization
in both selection rules, Algorithm~\ref{alg:pc_submax} deterministically
returns a feasible set $Q$ such that, for every feasible comparator
$O\subseteq\mathcal V$,
$U(Q)-\ell(Q)\geq\max\left\{0,\;\frac12 U(O)-\ell(O)\right\}.
\label{eq:rgm_guarantee}$
The algorithm uses $O(n\kappa)$ value-oracle queries, where $n=|\mathcal V|$ and $\kappa$ is
the maximum cardinality of a feasible set.  In particular, the result
allows the modular cost $\ell_e$ and knapsack weight $c_e$ to vary
independently; the token price $\ell_e=\lambda_{\mathrm{tok}}c_e/B$
used by PC-SubMax is a special case.
\end{theorem}
Proofs of the component properties and the RGM guarantee are provided in Appendix~\ref{sec:appendix_proofs}.
}

\cs{
\paragraph{Accelerated variant.}
Appendix~\ref{sec:fast_rgm} gives a lazy implementation of the same cost-scaled path and per-prefix augmentation framework. With exact marginal evaluations but approximate selection rules, it achieves a $(1/2-\varepsilon,1)$ regularized guarantee using $O((n/\varepsilon)[1+\log(B/(\varepsilon c_{\min}))])$ value-oracle queries, where $0<\varepsilon<1/2$ and $c_{\min}$ is the minimum individually feasible item cost. If no item is individually feasible, the selector returns $\emptyset$ immediately. Preprocessing and marginal-evaluation costs are analyzed below. Our empirical evaluation uses the full-scan RGM in Algorithm~\ref{alg:pc_submax}.
}

\subsection{Computational Complexity Analysis}

\cs{
For the full-scan implementation, let $T_{\mathrm{prep}}$ denote the total cost of text preprocessing, token-cost and vocabulary construction, encoding the sentences and the base query, and optional multi-hop relevance computation. Dense kernel construction is counted separately and costs $O(n^2d)$. At a prefix of size $j$, a Cholesky-based triangular solve evaluates one diversity marginal in $O(1+j^2)$ operations. Coverage evaluation costs $O(|W_e|)$ per candidate under constant-time dictionary membership operations, while the relevance and penalty marginals are constant-time lookups. Summing the candidate scans yields $O(n\kappa\bar L+n\kappa^3)$ selection work, where $\bar L=n^{-1}\sum_i|W_i|$. Detailed derivations appear in Appendix~\ref{sec:appendix_complexity}.
}

\begin{theorem}\label{thm:complexity}
\cs{
For a candidate set with at least one individually feasible item, let $n$ be the number of sentences, $\kappa$ the maximum feasible cardinality, $d$ the embedding dimension, and $\bar L$ the average number of unique words per sentence. With preprocessing cost $T_{\mathrm{prep}}$ defined above and constant-time dictionary membership operations, the dense full-scan implementation has runtime
$O(T_{\mathrm{prep}}+n^2d+n\kappa\bar L+n\kappa^3).$
Under bounded sentence lengths and fixed encoder configuration, encoder input limits, hop limit, and beam width, this simplifies to $O(n^2d+n\kappa^3)$.
}
\end{theorem}

\section{Experiments}
\cs{We evaluate the downstream performance and compression overhead of PC-SubMax across multiple tasks, token budgets, and inference models. The experiments assess the usefulness of the selected prompts, while component ablations examine how coverage, diversity, and relevance contribute to downstream performance. Timing comparisons measure the online cost of compression.}

\subsection{Experimental Setup}
\label{sec:experiments}

\textbf{Datasets.} Following PartPrompt's task settings~\citep{mao2025parse}, we evaluate seven benchmarks: text and code summarization (Truncated ArXiv, People Daily, CodeNet~\citep{puri2021codenet}); mathematical reasoning (GSM8K~\citep{cobbe2021training}); multi-hop QA (HotpotQA~\citep{yang2018hotpotqa}); and long-context understanding (LongBench~\citep{bai2024longbench}, RULER~\citep{hsieh2024ruler}).

\textbf{Baselines.} We compare five \cs{representative prompt-compression baselines with publicly available implementations}. Our test suite includes Selective-Context~\citep{li2023compressing} (information-theoretic filtering), LLMLingua~\citep{jiang2023llmlingua} and LongLLMLingua~\citep{jiang2024longllmlingua} (perplexity-driven compression), LLMLingua-2~\citep{pan2024llmlingua2} (token-level classification via knowledge distillation), and PartPrompt~\citep{mao2025parse} (hierarchical structure-aware pruning).

\textbf{Evaluation Metrics.} Following PartPrompt~\citep{mao2025parse}, we report BLEU, ROUGE-1/2/L, and BERTScore-F1 (BS-F1) for Truncated ArXiv, People Daily, and CodeNet summarization. HotpotQA additionally uses answer precision, recall, and F1; GSM8K uses Exact Match (EM); LongBench and RULER use aggregate evaluation scores.

\textbf{Implementation Details.}\footnote{Our code and reproduction instructions are available at an anonymous repository: \url{https://anonymous.4open.science/r/PC-SubMax-89DD14321}.} \cs{Qwen3-VL-32B-Instruct is the default downstream LLM; the cross-model experiment additionally uses Llama-3-8B-Instruct. Sentence embeddings are obtained from the pretrained \texttt{multilingual-e5-small} encoder and normalized to form the Gram kernel $K_{ij}=z_i^\top z_j$. Table~\ref{tab:weights} lists the coefficients of the utility components and token penalty. The task-specific weights are selected by grid search on the validation split and are fixed across all test instances and compression ratios for each task. We use multi-hop relevance with $H=2$ on HotpotQA and single-hop relevance elsewhere when $\lambda_{\mathrm{rel}}>0$; configurations with $\lambda_{\mathrm{rel}}=0$ do not use relevance in the objective. Compression uses the full-scan RGM in Algorithm~\ref{alg:pc_submax}, with fixed tie-breaking by sentence index. All compression runs are performed on a single NVIDIA A100 GPU with 40GB of memory. Default downstream inference accesses Qwen3-VL-32B-Instruct through an API.}

\subsection{Experimental Results}
\cs{We examine compression quality across tasks, performance at different token-retention ratios, compression time, transfer across two inference models, and component ablations. Additional results appear in Appendix~\ref{sec:appendix_summarization}. Across the tables, \textit{CR} denotes the requested fraction of tokens retained, so CR$=0.2$ requests retention of 20\% of the original tokens. The reported $1/\tau$ is the achieved length-reduction factor (original length divided by compressed length), and \textit{Tokens} is the average compressed length. Larger $1/\tau$ therefore indicates stronger compression. We report achieved lengths alongside target ratios to make differences in token usage visible.}

\begin{table}[!htbp]
\caption{Comprehensive results on CodeNet and GSM8K at different compression ratios. Best results are shown in \textbf{bold}.}
\label{tab:code_gsm8k}
\centering
\footnotesize
\resizebox{\textwidth}{!}{
\begin{tabular}{@{}l c  ccccccc ccc@{}}
\toprule
\multirow{2}{*}{\textbf{Method}} & \multirow{2}{*}{\textbf{CR}}
 & \multicolumn{7}{c}{\textbf{CodeNet}} & \multicolumn{3}{c}{\textbf{GSM8K}} \\
\cmidrule(lr){3-9} \cmidrule(lr){10-12}
 & & \textbf{$1/\tau$} & \textbf{Tokens} & \textbf{BLEU} & \textbf{R-1} & \textbf{R-2} & \textbf{R-L} & \textbf{BS-F1}
   & \textbf{EM} & \textbf{$1/\tau$} & \textbf{Tokens} \\
\midrule
\multirow{3}{*}{Selective-Context}
 & 0.2 & 4.53 & 335.10 & 3.21  & 25.42 & 7.04  & 18.93 & 83.34 & 87.06 & 4.28 & 114 \\
 & 0.3 & 3.27 & 464.22 & 5.53  & 32.00 & 8.26  & 17.90 & 82.93 & 92.16 & 2.84 & 172 \\
 & 0.5 & 1.89 & 803.17 & 11.46 & 41.72 & 14.63 & 23.99 & 84.67 & 90.20 & 1.60 & 305 \\
\midrule
\multirow{3}{*}{LLMLingua}
 & 0.2 & 5.02 & 302.39 & 4.96  & 31.50 & 8.10  & 17.47 & 82.24 & 86.87 & 6.18 & 79  \\
 & 0.3 & 3.41 & 445.16 & 6.72  & 34.23 & 10.03 & 20.06 & 83.38 & 85.22 & 3.87 & 126 \\
 & 0.5 & 2.00 & 759.00 & 11.99 & 42.18 & 16.11 & 24.73 & 85.24 & 87.10 & 2.28 & 214 \\
\midrule
\multirow{3}{*}{LongLLMLingua}
 & 0.2 & 5.23 & 290.25 & 3.89  & 29.71 & 7.08  & 15.92 & 82.75 & 81.57 & 5.81 & 84  \\
 & 0.3 & 3.27 & 464.22 & 6.78  & 33.70 & 8.63  & 18.29 & 83.23 & 81.17 & 3.75 & 130 \\
 & 0.5 & 2.05 & 740.49 & 13.32 & 42.58 & 14.51 & 24.60 & 84.67 & 87.45 & 2.25 & 217 \\
\midrule
\multirow{3}{*}{LLMLingua-2}
 & 0.2 & 4.90 & 309.80 & 3.06  & 27.17 & 5.68  & 14.22 & 81.96 & 76.47 & 5.25 & 93  \\
 & 0.3 & 3.29 & 461.40 & 6.21  & 33.25 & 9.68  & 18.39 & 83.11 & 80.78 & 3.49 & 140 \\
 & 0.5 & 2.02 & 751.49 & 12.71 & 45.41 & 15.98 & 24.86 & 85.75 & 87.45 & 2.05 & 238 \\
\midrule
\multirow{3}{*}{PartPrompt}
 & 0.2 & 4.89 & 310.43 & 6.08  & 30.54 & 8.83  & 18.22 & 83.57 & 87.90 & 4.74 & 103 \\
 & 0.3 & 3.21 & 472.90 & 8.62  & 34.40 & 11.66 & 23.15 & 85.99 & 87.12 & 3.13 & 156 \\
 & 0.5 & 1.93 & 786.53 & 14.57 & 40.53 & 22.39 & 32.41 & 87.48 & 86.50 & 1.89 & 258 \\
\midrule
\multirow{3}{*}{\textbf{PC-SubMax (Ours)}}
 & 0.2 & 4.61 & 329.28 & \textbf{7.91}  & \textbf{36.33} & \textbf{10.73} & \textbf{20.59} & \textbf{85.60} & \textbf{91.77} & 5.08 & 96.06 \\
 & 0.3 & 3.12 & 486.54 & \textbf{10.16} & \textbf{39.94} & \textbf{13.40} & \textbf{23.25} & \textbf{87.08} & \textbf{94.90} & 3.34 & 146.11 \\
 & 0.5 & 1.96 & 774.50 & \textbf{17.44} & \textbf{47.38} & \textbf{24.96} & \textbf{34.99} & \textbf{90.92} & \textbf{94.12} & 2.01 & 242.79 \\
\bottomrule
\end{tabular}
}
\end{table}

\paragraph{Compression Quality across Tasks}
We evaluate PC-SubMax on seven diverse datasets across four distinct task categories. Table~\ref{tab:code_gsm8k} \cs{reports results on CodeNet and GSM8K. PC-SubMax has the highest reported quality scores among the compared methods at each target ratio in this table. Its GSM8K EM exceeds PartPrompt by 3.87, 7.78, and 7.62 percentage points at CR$=0.2$, $0.3$, and $0.5$, respectively. On CodeNet at CR$=0.5$, the BLEU improvement over PartPrompt is 2.87 points.} Results for all seven datasets are reported in Table~\ref{tab:code_gsm8k} and Tables~\ref{tab:arxiv_results}--\ref{tab:longbench_ruler_main}. \cs{These results support the empirical effectiveness of the selected prompts across the evaluated tasks.}

\paragraph{Robustness across Compression Ratios}
\cs{Table~\ref{tab:arxiv_results} reports Truncated ArXiv results at target retention ratios of 0.2, 0.3, and 0.5; Figure~\ref{fig:ratio_analysis} presents the compression-ratio comparison. Across the three tabulated ratios, PC-SubMax's BLEU, ROUGE, and BERTScore increase as more tokens are retained. It has the highest reported scores except for ROUGE-L at CR$=0.3$, where PartPrompt is ahead by 0.06 points. These observations describe the measured trade-off between compressed length and downstream performance.}

\begin{figure}[!t]
\centering
\safeincludegraphics[width=1.0\linewidth]{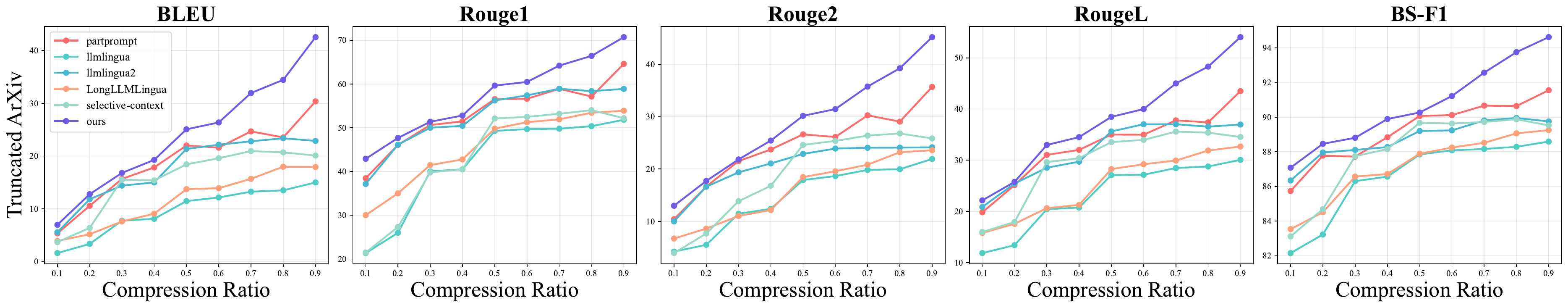}
\caption{Performance comparison on Truncated ArXiv across different compression ratios.}
\label{fig:ratio_analysis}
\end{figure}

\paragraph{Compression Efficiency}
\cs{Table~\ref{tab:time} reports wall-clock compression time on six benchmarks. PC-SubMax is the fastest of the evaluated methods on four benchmarks and is faster than the perplexity-based baselines wherever their timings are available. Relative to PartPrompt, the arithmetic mean of the six per-benchmark speedup ratios is $34.8\times$, ranging from $1.5\times$ on HotpotQA to $124.1\times$ on RULER. LLMLingua-2 is faster on LongBench and RULER by factors of $1.7$ and $1.8$, respectively. PC-SubMax achieves these compression times using pretrained encoder representations and incremental subset selection, without compression-specific model training or causal-language-model scoring.}
\begin{table}[!htbp]
\caption{\cs{Wall-clock compression time in seconds for the evaluated workloads.}}
\label{tab:time}
\centering
\begingroup
\footnotesize
\resizebox{\linewidth}{!}
{
\begin{tabular}{l cccccc}
\toprule
\textbf{Method} & \textbf{Truncated ArXiv} & \textbf{HotpotQA} & \textbf{CodeNet} & \textbf{People Daily} & \textbf{LongBench} & \textbf{RULER} \\
\midrule
Selective-Context & 3046 & 9173 & 485 & 635 & -- & -- \\
LLMLingua        & 2235 & 5546 & 1033 & 1901 & -- & -- \\
LongLLMLingua    & 2268 & 5532 & 1042 & 1912 & -- & -- \\
LLMLingua-2      & 387 & 1045 & 181 & 134 & 367 & 321 \\
PartPrompt       & 3101 & 1304 & 2138 & 1267 & 14024 & 71491 \\
PC-SubMax (Ours) & 154 & 897 & 82 & 85 & 627 & 576 \\
\bottomrule
\end{tabular}
}
\endgroup
\end{table}

\paragraph{Cross-LLM Generalization}
\cs{We evaluate compressed prompts with Qwen3-VL-32B-Instruct and Llama-3-8B-Instruct on Truncated ArXiv. As reported in Table~\ref{tab:different_llms} of Appendix~\ref{sec:app_llm}, PC-SubMax has the highest scores on all five metrics at all three target ratios for Llama-3-8B-Instruct. For Qwen3-VL-32B-Instruct, the only exception is ROUGE-L at CR$=0.3$. The results support transfer of the compression method across these two model families in the evaluated summarization setting.}

\paragraph{Ablation Study}
\cs{We vary the three semantic-utility components on HotpotQA while keeping $\lambda_{\mathrm{tok}}=0.05$ fixed. Table~\ref{tab:hotpotqa_ablation} in Appendix~\ref{sec:app_ablation} shows that the full configuration achieves the highest answer F1 at every target ratio. Removing relevance yields the largest F1 reductions: 22.09, 31.36, and 30.89 points at CR$=0.2$, $0.3$, and $0.5$, respectively. Removing coverage reduces F1 by 1.50--6.09 points, while removing diversity reduces it by 0.64--4.65 points. These comparisons support the complementary roles of the utility components, with query relevance contributing most strongly on this benchmark.}

\section{Conclusion}

\cs{We presented PC-SubMax, a sentence-level prompt-compression framework based on regularized monotone submodular maximization under a knapsack constraint. Its objective balances lexical coverage, embedding diversity, and query relevance against a modular token penalty. The deterministic RGM algorithm achieves a $(1/2,1)$ regularized approximation guarantee with $O(n\kappa)$ value-oracle queries, and its lazy variant achieves a $(1/2-\varepsilon,1)$ guarantee with $O((n/\varepsilon)[1+\log(B/(\varepsilon c_{\min}))])$ queries. Experiments with full-scan RGM show competitive downstream performance across seven benchmarks and the lowest measured compression time on four of the six timing benchmarks. Future work will investigate relevance representations for broader tasks and implementations that reduce the cost of dense kernel construction and diversity-marginal evaluation.}

\clearpage

\bibliographystyle{plainnat}
\bibliography{references}

@article{brown2020language,
  title={Language models are few-shot learners},
  author={Brown, Tom and Mann, Benjamin and Ryder, Nick and Subbiah, Melanie and Kaplan, Jared D and Dhariwal, Prafulla and Neelakantan, Arvind and Shyam, Pranav and Sastry, Girish and Askell, Amanda and others},
  journal={Advances in Neural Information Processing Systems (NeurIPS)},
  volume={33},
  pages={1877--1901},
  year={2020}
}

@article{gong2024budget,
  title={Budget-constrained profit maximization without non-negative objective assumption in social networks},
  author={Gong, Suning and Nong, Qingqin and Wang, Yue and Du, Dingzhu},
  journal={Journal of Global Optimization (JGO)},
  volume={90},
  number={4},
  pages={1007--1030},
  year={2024}
}

@article{guo2026efficient,
  title={Efficient Algorithms for Budgeted Profit Maximization With Theoretical Guarantees},
  author={Guo, Qintian and Feng, Chen and Shi, Jieming and Tang, Jing and Zhou, Xiaofang and Wang, Sibo},
  journal={IEEE Transactions on Knowledge and Data Engineering (TKDE)},
  year={2026}
}

@inproceedings{vaswani2017attention,
  title={Attention is all you need},
  author={Vaswani, Ashish and Shazeer, Noam and Parmar, Niki and Uszkoreit, Jakob and Jones, Llion and Gomez, Aidan N and Kaiser, Lukasz and Polosukhin, Illia},
  booktitle={Advances in Neural Information Processing Systems (NeurIPS)},
  pages={5998--6008},
  year={2017}
}

@article{liu2024lost,
  title={Lost in the middle: How language models use long contexts},
  author={Liu, Nelson F and Lin, Kevin and Hewitt, John and Paranjape, Ashwin and Bevilacqua, Michele and Petroni, Fabio and Liang, Percy},
  journal={Transactions of the Association for Computational Linguistics (TACL)},
  volume={12},
  pages={157--173},
  year={2024}
}

@inproceedings{shi2023large,
  title={Large language models can be easily distracted by irrelevant context},
  author={Shi, Freda and Chen, Xinyun and Misra, Kanishka and Scales, Nathan and Dohan, David and Chi, Ed H and Sch{\"a}rli, Nathanael and Zhou, Denny},
  booktitle={International Conference on Machine Learning (ICML)},
  pages={31210--31227},
  year={2023}
}

@inproceedings{li2023compressing,
  title={Compressing context to enhance inference efficiency of large language models},
  author={Li, Yucheng and Dong, Bo and Guerin, Frank and Lin, Chenghua},
  booktitle={Conference on Empirical Methods in Natural Language Processing (EMNLP)},
  pages={6342--6353},
  year={2023}
}

@inproceedings{jiang2023llmlingua,
  title={Llmlingua: Compressing prompts for accelerated inference of large language models},
  author={Jiang, Huiqiang and Wu, Qianhui and Lin, Chin-Yew and Yang, Yuqing and Qiu, Lili},
  booktitle={Conference on Empirical Methods in Natural Language Processing (EMNLP)},
  pages={13358--13376},
  year={2023}
}

@inproceedings{jiang2024longllmlingua,
  title={Longllmlingua: Accelerating and enhancing llms in long context scenarios via prompt compression},
  author={Jiang, Huiqiang and Wu, Qianhui and Luo, Xufang and Li, Dongsheng and Lin, Chin-Yew and Yang, Yuqing and Qiu, Lili},
  booktitle={Annual Meeting of the Association for Computational Linguistics (ACL)},
  pages={1658--1677},
  year={2024}
}

@inproceedings{pan2024llmlingua2,
  title={Llmlingua-2: Data distillation for efficient and faithful task-agnostic prompt compression},
  author={Pan, Zhuoshi and Wu, Qianhui and Jiang, Huiqiang and Xia, Menglin and Luo, Xufang and Zhang, Jue and Lin, Qingwei and R{\"u}hle, Victor and Yang, Yuqing and Lin, Chin-Yew and others},
  booktitle={Findings of the Association for Computational Linguistics (Findings of ACL)},
  pages={963--981},
  year={2024}
}

@article{mao2025parse,
  title={Parse trees guided LLM prompt compression},
  author={Mao, Wenhao and Hou, Chengbin and Zhang, Tianyu and Lin, Xinyu and Tang, Ke and Lv, Hairong},
  journal={IEEE Transactions on Pattern Analysis and Machine Intelligence (TPAMI)},
  year={2025}
}

@inproceedings{chuang2024learning,
  title={Learning to compress prompt in natural language formats},
  author={Chuang, Yu-Neng and Xing, Tianwei and Chang, Chia-Yuan and Liu, Zirui and Chen, Xun and Hu, Xia},
  booktitle={Conference of the North American Chapter of the Association for Computational Linguistics: Human Language Technologies (NAACL-HLT)},
  pages={7756--7767},
  year={2024}
}

@article{mu2023learning,
  title={Learning to compress prompts with gist tokens},
  author={Mu, Jesse and Li, Xiang and Goodman, Noah},
  journal={Advances in Neural Information Processing Systems (NeurIPS)},
  volume={36},
  pages={19327--19352},
  year={2023}
}

@article{ge2023context,
  title={In-context autoencoder for context compression in a large language model},
  author={Ge, Tao and Hu, Jing and Wang, Lei and Wang, Xun and Chen, Si-Qing and Wei, Furu},
  journal={arXiv preprint arXiv:2307.06945},
  year={2023}
}

@inproceedings{lin2011class,
  title={A class of submodular functions for document summarization},
  author={Lin, Hui and Bilmes, Jeff},
  booktitle={Annual Meeting of the Association for Computational Linguistics (ACL)},
  pages={510--520},
  year={2011}
}

@article{kulesza2012determinantal,
  title={Determinantal point processes for machine learning},
  author={Kulesza, Alex and Taskar, Ben},
  journal={Foundations and Trends in Machine Learning},
  volume={5},
  number={2-3},
  pages={123--286},
  year={2012}
}

@inproceedings{mirzasoleiman2016fast,
  title={Fast constrained submodular maximization: Personalized data summarization},
  author={Mirzasoleiman, Baharan and Badanidiyuru, Ashwinkumar and Karbasi, Amin},
  booktitle={International Conference on Machine Learning (ICML)},
  pages={1358--1367},
  year={2016}
}

@inproceedings{wei2015submodularity,
  title={Submodularity in data subset selection and active learning},
  author={Wei, Kai and Iyer, Rishabh and Bilmes, Jeff},
  booktitle={International Conference on Machine Learning (ICML)},
  pages={1954--1963},
  year={2015}
}

@inproceedings{iyer2021submodular,
  title={Submodular combinatorial information measures with applications in machine learning},
  author={Iyer, Rishabh and Khargoankar, Ninad and Bilmes, Jeff and Asanani, Himanshu},
  booktitle={Algorithmic Learning Theory (ALT)},
  pages={722--754},
  year={2021}
}

@inproceedings{yang2018hotpotqa,
  title={HotpotQA: A dataset for diverse, explainable multi-hop question answering},
  author={Yang, Zhilin and Qi, Peng and Zhang, Saizheng and Bengio, Yoshua and Cohen, William and Salakhutdinov, Ruslan and Manning, Christopher D},
  booktitle={Conference on Empirical Methods in Natural Language Processing (EMNLP)},
  pages={2369--2380},
  year={2018}
}

@inproceedings{bai2024longbench,
  title={Longbench: A bilingual, multitask benchmark for long context understanding},
  author={Bai, Yushi and Lv, Xin and Zhang, Jiajie and Lyu, Hongchang and Tang, Jiankai and Huang, Zhidian and Du, Zhengxiao and Liu, Xiao and Zeng, Aohan and Hou, Lei and others},
  booktitle={Annual Meeting of the Association for Computational Linguistics (ACL)},
  pages={3119--3137},
  year={2024}
}

@article{hsieh2024ruler,
  title={RULER: What's the real context size of your long-context language models?},
  author={Hsieh, Cheng-Ping and Sun, Simeng and Kriman, Samuel and Acharya, Shantanu and Rekesh, Dima and Jia, Fei and Zhang, Yang and Ginsburg, Boris},
  journal={arXiv preprint arXiv:2404.06654},
  year={2024}
}

@article{cobbe2021training,
  title={Training verifiers to solve math word problems},
  author={Cobbe, Karl and Kosaraju, Vineet and Bavarian, Mohammad and Chen, Mark and Jun, Heewoo and Kaiser, Lukasz and Plappert, Matthias and Tworek, Jerry and Hilton, Jacob and Nakano, Reiichiro and others},
  journal={arXiv preprint arXiv:2110.14168},
  year={2021}
}

@article{puri2021codenet,
  title={Codenet: A large-scale ai for code dataset for learning a diversity of coding tasks},
  author={Puri, Ruchir and Kung, David S and Janssen, Geert and Zhang, Wei and Domeniconi, Giacomo and Zolotov, Vladimir and Dolby, Julian and Chen, Jie and Choudhury, Mihir and Decker, Lindsey and others},
  journal={arXiv preprint arXiv:2105.12655},
  year={2021}
}

@article{tang2026comi,
  title={COMI: Coarse-to-fine Context Compression via Marginal Information Gain},
  author={Tang, Jiwei and Liu, Shilei and Zhang, Zhicheng and Yuan, Yujin and Zheng, Libin and Su, Wenbo and Zheng, Bo},
  journal={arXiv preprint arXiv:2602.01719},
  year={2026}
}

@article{hu2026dynamic,
  title={Dynamic Prompt Compression for Efficient Inference of Large Language Models},
  author={Hu, Jinwu and Zhang, Wei and Wang, Yufeng and Hu, Yu and Xiao, Bin and Tan, Mingkui and Du, Qing},
  journal={IEEE Transactions on Knowledge and Data Engineering (TKDE)},
  year={2026}
}

@article{tang2026read,
  title={Read As Human: Compressing Context via Parallelizable Close Reading and Skimming},
  author={Tang, Jiwei and Liu, Shilei and Zhang, Zhicheng and Lv, Qingsong and Zhao, Runsong and Lu, Tingwei and Liu, Langming and Chen, Haibin and Yuan, Yujin and Zheng, Hai-Tao and others},
  journal={arXiv preprint arXiv:2602.01840},
  year={2026}
}

@article{yang2026less,
  title={Less is more: Docstring compression in code generation},
  author={Yang, Guang and Zhou, Yu and Cheng, Wei and Zhang, Xiangyu and Chen, Xiang and Zhuo, Terry Yue and Liu, Ke and Zhou, Xin and Lo, David and Chen, Taolue},
  journal={ACM Transactions on Software Engineering and Methodology (TOSEM)},
  volume={35},
  number={2},
  pages={1--31},
  year={2026}
}

@article{atandoh2026less,
  title={Less is More: Adaptive Prompt Compression and Exemplar Selection for Efficient Few-Shot Sentiment Analysis},
  author={Atandoh, Peter and Li, Yongkang and Guo, Weikang and Guo, Jinyu and Zou, Jie},
  journal={Expert Systems with Applications},
  pages={132138},
  year={2026}
}

@article{wang2026alleviating,
  title={Alleviating Contextual Misguidance: Response-Aware Prompt Compression for Long-Context Question Answering},
  author={Wang, Haoyuan and Wang, Zhen and Zhou, Wenmeng and Deng, Yang},
  journal={IEEE/ACM Transactions on Audio, Speech, and Language Processing (TASLP)},
  year={2026}
}

@article{lewis2020retrieval,
  title={Retrieval-augmented generation for knowledge-intensive nlp tasks},
  author={Lewis, Patrick and Perez, Ethan and Piktus, Aleksandra and Petroni, Fabio and Karpukhin, Vladimir and Goyal, Naman and K{\"u}ttler, Heinrich and Lewis, Mike and Yih, Wen-tau and Rockt{\"a}schel, Tim and others},
  journal={Advances in Neural Information Processing Systems (NeurIPS)},
  volume={33},
  pages={9459--9474},
  year={2020}
}

@article{ouyang2022training,
  title={Training language models to follow instructions with human feedback},
  author={Ouyang, Long and Wu, Jeffrey and Jiang, Xu and Almeida, Diogo and Wainwright, Carroll and Mishkin, Pamela and Zhang, Chong and Agarwal, Sandhini and Slama, Katarina and Ray, Alex and others},
  journal={Advances in Neural Information Processing Systems (NeurIPS)},
  volume={35},
  pages={27730--27744},
  year={2022}
}

@article{touvron2023llama,
  title={Llama: Open and efficient foundation language models},
  author={Touvron, Hugo and Lavril, Thibaut and Izacard, Gautier and Martinet, Xavier and Lachaux, Marie-Anne and Lacroix, Timoth{\'e}e and Rozi{\`e}re, Baptiste and Goyal, Naman and Hambro, Eric and Azhar, Faisal and others},
  journal={arXiv preprint arXiv:2302.13971},
  year={2023}
}

@article{lin2026docsage,
  title={DocSage: An Information Structuring Agent for Multi-Doc Multi-Entity Question Answering},
  author={Lin, Teng and Zhu, Yizhang and Zhang, Zhengxuan and Luo, Yuyu and Tang, Nan},
  journal={arXiv preprint arXiv:2603.11798},
  year={2026}
}

@inproceedings{yaroslavtsev2020bring,
  title={“bring your own greedy”+ max: near-optimal 1/2-approximations for submodular knapsack},
  author={Yaroslavtsev, Grigory and Zhou, Samson and Avdiukhin, Dmitrii},
  booktitle={International Conference on Artificial Intelligence and Statistics},
  pages={3263--3274},
  year={2020},
  organization={PMLR}
}

@inproceedings{nikolakaki2021efficient,
  title={An efficient framework for balancing submodularity and cost},
  author={Nikolakaki, Sofia Maria and Ene, Alina and Terzi, Evimaria},
  booktitle={Proceedings of the 27th ACM SIGKDD Conference on Knowledge Discovery \& Data Mining},
  pages={1256--1266},
  year={2021}
}

@inproceedings{liskavets2025prompt,
  title     = {Prompt Compression with Context-Aware Sentence Encoding for Fast and Improved {LLM} Inference},
  author    = {Liskavets, Barys and Ushakov, Maxim and Roy, Shuvendu and Klibanov, Mark and Etemad, Ali and Luke, Shane K.},
  booktitle = {Proceedings of the AAAI Conference on Artificial Intelligence},
  volume    = {39},
  number    = {23},
  pages     = {24595--24604},
  year      = {2025},
  doi       = {10.1609/aaai.v39i23.34639}
}

\clearpage
\appendix
\section*{Appendix}

\section{Additional Related Work}

\subsection{Prompt Compression Techniques}
\cs{Prompt compression reduces the length of inputs supplied to an LLM. A useful distinction between white-box and black-box methods concerns their access requirements for the target model.}

\paragraph{White-box Prompt Compression.}
\cs{Methods that encode context into model-specific soft prompts require an interface that accepts continuous representations or modifications to the target model~\citep{mu2023learning,ge2023context,tang2026comi}. COMI~\citep{tang2026comi} uses marginal information gain to account for query relevance and redundancy, while RAM~\citep{tang2026read} combines detailed processing of relevant passages with compressed representations of background context. Such representations generally require access beyond a standard text-only interface. Another line of work uses response-dependent signals for selection: LOOC~\citep{wang2026alleviating} estimates the influence of prompt segments on model outputs to identify useful content. Its access requirements depend on the response signals used by the implementation.}

\paragraph{Black-box Prompt Compression.}
\cs{Black-box compressors produce text that can be passed to the target model through its usual input interface. Generative methods rewrite or summarize the context~\citep{chuang2024learning}; selective methods retain source tokens, phrases, or sentences. Rewriting can change factual details, whereas extraction preserves the wording of the retained content.}

\cs{Among selective methods, Selective-Context~\citep{li2023compressing} uses self-information, LLMLingua~\citep{jiang2023llmlingua} uses language-model scores and iterative compression, and LLMLingua-2~\citep{pan2024llmlingua2} uses a contextual token classifier. PartPrompt~\citep{mao2025parse} incorporates linguistic structure through hierarchical tree construction and budgeted pruning. LLM-DPC~\citep{hu2026dynamic} learns a sequential compression policy using PPO. Query-aware methods include LongLLMLingua~\citep{jiang2024longllmlingua}, which uses contrastive perplexity, and CPC~\citep{liskavets2025prompt}, which scores sentence relevance with a contrastively trained context-aware encoder. ShortenDoc~\citep{yang2026less} and SR-C3~\citep{atandoh2026less} study compression tailored to code generation and few-shot sentiment analysis, respectively.}

\subsection{Submodular Optimization}

Submodular functions naturally capture the property of diminishing marginal returns: the marginal gain of adding an element \cs{does not increase} as the set grows. This property makes submodular optimization a principled framework for combinatorial subset selection problems under resource constraints, with broad applications spanning active data selection~\citep{wei2015submodularity}, core-set extraction~\citep{iyer2021submodular}, extractive document summarization~\citep{lin2011class}, and diversity-aware subset selection via Determinantal Point Processes (DPPs)~\citep{kulesza2012determinantal} and facility-location functions~\citep{mirzasoleiman2016fast}. For non-negative monotone submodular maximization under a knapsack constraint, Greedy+Max augments every partial density-greedy solution with its best additional feasible item and obtains a one-half approximation with $O(n\kappa)$ oracle queries~\citep{yaroslavtsev2020bring}. A related line of work studies regularized objectives $U(S)-\ell(S)$, where $U$ is monotone submodular and $\ell$ is non-negative modular; cost-scaled greedy methods use marginals of the form $U(e\mid S)-2\ell_e$ and provide guarantees that separately retain a fraction of utility and subtract the full cost~\citep{nikolakaki2021efficient}. \cs{Under a knapsack constraint, PartEnum-Twin-Greedy provides a $(1/4,1)$ guarantee with $O(n^4)$ value-oracle queries~\citep{gong2024budget}. More recently, GEAR provides a $((1-1/e)/2,1/2)$ guarantee for budgeted profit maximization when the modular penalty also defines budget consumption~\citep{guo2026efficient}. Its utility and penalty coefficients express a different trade-off from RGM: RGM retains a larger utility coefficient, whereas GEAR subtracts a smaller fraction of the comparator penalty. RGM establishes a $(1/2,1)$ guarantee for arbitrary non-negative modular penalties and independent positive knapsack weights.}

\section{Additional Experimental Results}
\label{sec:appendix_summarization}

\cs{This appendix reports the task-specific weights, long-context results, additional summarization and question-answering results, cross-model comparisons, and semantic-component ablations that accompany the main evaluation.}

\subsection{Weight Configuration}\label{sec:app_weights} \cs{Table~\ref{tab:weights} lists the coefficients of the three utility components and the token penalty. Each dataset or task type uses the same listed weights across the reported compression ratios. The summarization configurations set $\lambda_{\mathrm{rel}}=0$, whereas question-answering configurations place greater weight on relevance.}

\begin{table}[!htbp]
\centering
\caption{\cs{Task-specific coefficients of the regularized objective.}}
\label{tab:weights}
\small
\setlength{\tabcolsep}{3pt}
{%
\begin{tabular}{@{}lcccc@{\qquad}lcccc@{}}
\toprule
\multicolumn{5}{c}{\textbf{Datasets}} & \multicolumn{5}{c}{\textbf{LongBench Tasks}} \\
\cmidrule(lr){1-5} \cmidrule(lr){6-10}
\textbf{Task} & $\lambda_{\mathrm{cov}}$ & $\lambda_{\mathrm{div}}$ & $\lambda_{\mathrm{rel}}$ & $\lambda_{\mathrm{tok}}$ &
\textbf{Task Type} & $\lambda_{\mathrm{cov}}$ & $\lambda_{\mathrm{div}}$ & $\lambda_{\mathrm{rel}}$ & $\lambda_{\mathrm{tok}}$ \\
\midrule
Truncated ArXiv & 0.50 & 0.50 & 0.00 & 0.10 & QA             & 0.25 & 0.10 & 0.60 & 0.05 \\
People Daily   & 0.50 & 0.50 & 0.00 & 0.05 & Summary        & 0.55 & 0.45 & 0.00 & 0.10 \\
CodeNet        & 0.50 & 0.50 & 0.00 & 0.05 & Retrieval      & 0.35 & 0.25 & 0.40 & 0.10 \\
HotpotQA       & 0.25 & 0.10 & 0.65 & 0.05 & Code           & 0.60 & 0.25 & 0.15 & 0.05 \\
GSM8K          & 0.50 & 0.35 & 0.15 & 0.05 & Classification & 0.35 & 0.15 & 0.50 & 0.10 \\
RULER          & 0.25 & 0.10 & 0.65 & 0.05 & Counting       & 0.60 & 0.25 & 0.15 & 0.10 \\
\bottomrule
\end{tabular}%
}
\end{table}


\subsection{Extended Summarization Benchmarks}

\paragraph{Performance on Truncated ArXiv.}
Table~\ref{tab:arxiv_results} presents the performance of different prompt compression methods on the Truncated ArXiv dataset. As shown in the table, PC-SubMax achieves the best result on every metric except ROUGE-L at CR$=0.3$, where PartPrompt is ahead by 0.06 points. This indicates that our method preserves the core semantic information of long academic texts well across compression levels.

\begin{table}[!htbp]
\centering
\caption{Results on Truncated ArXiv at different compression ratios. Best results are shown in \textbf{bold}.}
\label{tab:arxiv_results}
\footnotesize
\begin{tabular}{@{}l c c c c c c c c@{}}
\toprule
\textbf{Method} & \textbf{CR} & \textbf{$1/\tau$} & \textbf{Tokens} & \textbf{BLEU} & \textbf{R-1} & \textbf{R-2} & \textbf{R-L} & \textbf{BS-F1} \\
\midrule
\multirow{3}{*}{Selective-Context}
& 0.2 & 4.75 & 631.03 & 6.39 & 27.33 & 7.68 & 17.93 & 84.69 \\
& 0.3 & 2.99 & 1001.74 & 15.51 & 39.79 & 13.85 & 29.62 & 87.74 \\
& 0.5 & 1.63 & 1844.51 & 18.43 & 52.13 & 24.60 & 33.53 & 89.67 \\
\midrule
\multirow{3}{*}{LLMLingua}
& 0.2 & 5.11 & 587.35 & 3.36 & 26.01 & 5.53 & 13.40 & 83.22 \\
& 0.3 & 3.42 & 875.64 & 7.76 & 40.04 & 11.46 & 20.42 & 86.31 \\
& 0.5 & 2.04 & 1467.58 & 11.47 & 49.28 & 17.90 & 27.09 & 87.85 \\
\midrule
\multirow{3}{*}{LongLLMLingua}
& 0.2 & 5.02 & 596.78 & 5.19 & 35.03 & 8.61 & 17.57 & 84.51 \\
& 0.3 & 3.37 & 889.46 & 7.61 & 41.49 & 11.03 & 20.63 & 86.57 \\
& 0.5 & 2.00 & 1499.76 & 13.73 & 49.80 & 18.46 & 28.26 & 87.89 \\
\midrule
\multirow{3}{*}{LLMLingua-2}
& 0.2 & 5.06 & 592.88 & 11.82 & 46.13 & 16.62 & 25.44 & 87.96 \\
& 0.3 & 3.26 & 920.24 & 14.41 & 50.01 & 19.36 & 28.53 & 88.11 \\
& 0.5 & 1.99 & 1507.53 & 21.36 & 56.24 & 22.88 & 35.62 & 89.20 \\
\midrule
\multirow{3}{*}{PartPrompt}
& 0.2 & 5.08 & 590.55 & 10.58 & 46.09 & 16.64 & 25.13 & 87.78 \\
& 0.3 & 3.41 & 879.77 & 15.67 & 50.62 & 21.51 & \cs{\textbf{31.04}} & 87.72 \\
& 0.5 & 2.01 & 1492.53 & 22.01 & 56.55 & 26.61 & 35.01 & 90.07 \\
\midrule
\multirow{3}{*}{PC-SubMax (Ours)}
& 0.2 & 4.96 & 604.84 & \textbf{12.78} & \textbf{47.67} & \textbf{17.72} & \textbf{25.76} & \textbf{88.46} \\
& 0.3 & 3.31 & 906.34 & \textbf{16.81} & \textbf{51.37} & \textbf{21.82} & 30.98 & \textbf{88.81} \\
& 0.5 & 1.99 & 1507.54 & \textbf{25.08} & \textbf{59.64} & \textbf{30.15} & \textbf{38.47} & \textbf{90.27} \\
\bottomrule
\end{tabular}
\end{table}

\paragraph{Results on People Daily.} \cs{People Daily provides a Chinese-language summarization setting. Table~\ref{tab:people_daily} shows that PC-SubMax has the highest reported ROUGE and BERTScore values at all three target ratios. It also has the highest BLEU at CR$=0.2$ and $0.5$; at CR$=0.3$, LLMLingua-2 is ahead by 0.13 points. These results extend the empirical evaluation beyond English-language inputs.}

\begin{table}[!htbp]
\centering
\caption{Results on People Daily at different compression ratios. Best results are shown in \textbf{bold}.}
\label{tab:people_daily}
\footnotesize
\begin{tabular}{@{}l c c c c c c c c@{}}
\toprule
\textbf{Method} & \textbf{CR} & \textbf{\(1/\tau\)} & \textbf{Tokens} & \textbf{BLEU} & \textbf{R-1} & \textbf{R-2} & \textbf{R-L} & \textbf{BS-F1} \\
\midrule
\multirow{3}{*}{Selective-Context}
& 0.2 & 4.96 & 474.80 & 8.33 & 29.34 & 12.13 & 28.63 & 66.52 \\
& 0.3 & 3.25 & 724.62 & 8.60 & 30.71 & 11.98 & 28.64 & 67.92 \\
& 0.5 & 1.95 & 1207.69 & 12.94 & 31.77 & 16.59 & 32.26 & 71.20 \\
\midrule
\multirow{3}{*}{LLMLingua}
& 0.2 & 5.14 & 458.17 & 2.71 & 23.57 & 7.43 & 23.80 & 62.02 \\
& 0.3 & 3.42 & 688.60 & 8.31 & 28.94 & 10.88 & 29.42 & 68.34 \\
& 0.5 & 2.05 & 1148.78 & 11.04 & 30.66 & 14.67 & 29.92 & 68.43 \\
\midrule
\multirow{3}{*}{LongLLMLingua}
& 0.2 & 5.10 & 461.76 & 2.98 & 22.19 & 6.82 & 21.58 & 59.56 \\
& 0.3 & 3.34 & 705.09 & 7.07 & 28.20 & 9.17 & 26.20 & 66.02 \\
& 0.5 & 2.05 & 1148.78 & 10.83 & 32.45 & 14.33 & 30.60 & 71.11 \\
\midrule
\multirow{3}{*}{LLMLingua-2}
& 0.2 & 4.79 & 491.65 & 9.20 & 32.52 & 13.49 & 30.28 & 70.09 \\
& 0.3 & 3.33 & 707.21 & \textbf{14.80} & 38.98 & 16.49 & 34.53 & 71.69 \\
& 0.5 & 2.02 & 1165.84 & 20.75 & 43.56 & 25.67 & 41.15 & 77.07 \\
\midrule
\multirow{3}{*}{PartPrompt}
& 0.2 & 5.03 & 468.19 & 10.56 & 31.66 & 12.09 & 29.42 & 68.34 \\
& 0.3 & 3.31 & 711.48 & 14.33 & 36.07 & 16.05 & 34.49 & 71.24 \\
& 0.5 & 2.01 & 1171.64 & 21.39 & 42.67 & 24.42 & 41.05 & 75.82 \\
\midrule
\multirow{3}{*}{PC-SubMax (Ours)}
& 0.2 & 5.13 & 459.06 & \textbf{11.01} & \textbf{34.04} & \textbf{18.54} & \textbf{33.52} & \textbf{73.52} \\
& 0.3 & 3.40 & 692.65 & 14.67 & \textbf{39.83} & \textbf{21.91} & \textbf{39.89} & \textbf{75.86} \\
& 0.5 & 2.04 & 1154.41 & \textbf{23.91} & \textbf{49.08} & \textbf{27.33} & \textbf{47.81} & \textbf{79.03} \\
\bottomrule
\end{tabular}
\end{table}

\subsection{Performance on HotpotQA}
As shown in Table~\ref{tab:hotpotqa_multi}, PC-SubMax achieves the best reported performance under all evaluated compression ratios. \cs{At CR$=0.2$, it obtains an answer F1 of 56.43, exceeding LLMLingua-2 by 1.47 points and PartPrompt by 12.03 points. It uses 279.82 tokens on average, compared with 298.09 and 299.38 tokens for these two baselines, respectively. This configuration uses the precomputed multi-hop relevance scores
defined in Section~\ref{sc:method}.}

\begin{table}[!htbp]
\caption{Results on HotpotQA at different compression ratios. Best results are shown in \textbf{bold}. R-1/2/L denotes ROUGE-1/2/L; P/R/F1 denotes precision/recall/F1.}
\label{tab:hotpotqa_multi}
\centering
\begingroup
\small
\setlength{\tabcolsep}{2.5pt}
\renewcommand{\arraystretch}{1.05}
{
\begin{tabular}{l c cc cccccccc}
\toprule
\textbf{Method} & \textbf{CR} & \textbf{\(1/\tau\)} & \textbf{Tokens} & \textbf{BLEU} & \textbf{R-1} & \textbf{R-2} & \textbf{R-L} & \textbf{BS-F1} & \textbf{P} & \textbf{R} & \textbf{F1} \\
\midrule
\multirow{3}{*}{Selective-Context}
& 0.2 & 4.87 & 318.70 & 35.41 & 38.81 & 17.82 & 38.74 & 89.12 & 43.01 & 43.21 & 40.60 \\
& 0.3 & 3.22 & 481.65 & 39.36 & 43.27 & 21.48 & 43.17 & 89.81 & 47.62 & 48.02 & 45.18 \\
& 0.5 & 1.99 & 780.84 & 48.68 & 53.34 & 27.75 & 53.20 & 91.53 & 58.39 & 59.12 & 55.77 \\
\midrule
\multirow{3}{*}{LLMLingua}
& 0.2 & 5.03 & 308.81 & 33.94 & 37.29 & 16.75 & 37.21 & 89.19 & 40.87 & 41.35 & 38.82 \\
& 0.3 & 3.11 & 499.28 & 35.00 & 38.48 & 17.35 & 38.39 & 89.43 & 41.99 & 42.13 & 39.96 \\
& 0.5 & 1.98 & 786.18 & 38.57 & 42.39 & 19.51 & 42.28 & 90.04 & 46.61 & 46.49 & 44.19 \\
\midrule
\multirow{3}{*}{LongLLMLingua}
& 0.2 & 5.01 & 309.80 & 34.75 & 38.35 & 17.62 & 38.29 & 89.12 & 41.88 & 42.67 & 39.85 \\
& 0.3 & 3.20 & 484.67 & 35.92 & 39.60 & 18.25 & 39.51 & 89.46 & 43.34 & 43.75 & 41.18 \\
& 0.5 & 1.99 & 780.15 & 40.23 & 44.31 & 20.78 & 44.21 & 90.24 & 48.46 & 48.62 & 46.02 \\
\midrule
\multirow{3}{*}{LLMLingua-2}
& 0.2 & 5.21 & 298.09 & 48.01 & 52.61 & 25.52 & 52.55 & 91.74 & 58.83 & 55.46 & 54.96 \\
& 0.3 & 3.34 & 464.67 & 56.10 & 61.06 & 31.00 & 61.02 & 92.92 & 67.27 & 64.39 & 63.59 \\
& 0.5 & 1.99 & 779.28 & 63.97 & 69.09 & 37.85 & 69.06 & 94.16 & 74.92 & 73.48 & 71.89 \\
\midrule
\multirow{3}{*}{PartPrompt}
& 0.2 & 5.19 & 299.38 & 39.67 & 43.04 & 19.84 & 42.98 & 89.34 & 45.46 & 49.58 & 44.40 \\
& 0.3 & 3.41 & 455.43 & 44.22 & 48.33 & 23.21 & 48.29 & 89.91 & 50.36 & 56.07 & 49.85 \\
& 0.5 & 1.99 & 781.24 & 51.69 & 56.12 & 27.93 & 56.00 & 91.52 & 59.00 & 64.02 & 57.99 \\
\midrule
\multirow{3}{*}{PC-SubMax (Ours)}
& 0.2 & 5.55 & 279.82 & \textbf{51.18} & \textbf{55.00} & \textbf{29.27} & \textbf{54.84} & \textbf{92.17} & \textbf{59.00} & \textbf{57.41} & \textbf{56.43} \\
& 0.3 & 3.70 & 419.73 & \textbf{62.25} & \textbf{65.57} & \textbf{34.01} & \textbf{65.57} & \textbf{94.25} & \textbf{69.26} & \textbf{66.32} & \textbf{67.29} \\
& 0.5 & 2.13 & 729.12 & \textbf{67.12} & \textbf{71.82} & \textbf{39.56} & \textbf{71.82} & \textbf{94.77} & \textbf{75.18} & \textbf{74.19} & \textbf{73.24} \\
\bottomrule
\end{tabular}
}
\endgroup
\end{table}

\subsection{Long-Context Understanding Tasks}
\cs{Table~\ref{tab:longbench_ruler_main} compares PC-SubMax with LLMLingua-2 and PartPrompt at the same target retention ratio of 0.2. PC-SubMax's LongBench aggregate score is 40.86, exceeding PartPrompt's 28.05 by 12.81 points (45.7\% relative). On the reported RULER evaluation, the corresponding scores are 97.2 and 96.0. Selective-Context, LLMLingua, and LongLLMLingua are not included in these comparisons because their input-length constraints prevent them from processing the long-context inputs used in LongBench and RULER under the same evaluation setting.}
\begin{table}[!htbp]
\caption{\cs{Results on LongBench and RULER at a target token-retention ratio of 0.2.}}
\label{tab:longbench_ruler_main}
\centering
\begingroup
\footnotesize
\begin{tabular}{l cc}
\toprule
\textbf{Method} & \textbf{RULER} & \textbf{LongBench} \\
\midrule
LLMLingua-2      & 70.4 & 25.93 \\
PartPrompt       & 96 & 28.05 \\
PC-SubMax (Ours) & \textbf{97.2} & \textbf{40.86} \\
\bottomrule
\end{tabular}
\endgroup
\end{table}



\subsection{Using Different LLMs for Inference}\label{sec:app_llm}
To evaluate whether the proposed compressor generalizes across different backbone language models, we feed the compressed prompts into Qwen3-VL-32B-Instruct and Llama-3-8B-Instruct on the Truncated ArXiv dataset. Table~\ref{tab:different_llms} reports BLEU, ROUGE, and BERTScore results under different token constraints. PC-SubMax achieves the best result in all tested settings except ROUGE-L on Qwen3-VL-32B-Instruct under the 30\% constraint, where PartPrompt is 0.06 points ahead. \cs{This comparison supports the effectiveness of the compression method across the two evaluated model families.}

\begin{table}[!htbp]
\caption{Performance of selective prompt compression methods on the Truncated ArXiv dataset when feeding compressed prompts to different LLMs.}
\label{tab:different_llms}
\centering
\begingroup
\footnotesize
\resizebox{\linewidth}{!}{%
\begin{tabular}{l ccccc ccccc}
\toprule
\multirow{2}{*}{\textbf{Method}} &
\multicolumn{5}{c}{\textbf{Qwen3-VL-32B-Instruct}} &
\multicolumn{5}{c}{\textbf{Llama-3-8B-Instruct}} \\
\cmidrule(lr){2-6} \cmidrule(lr){7-11}
& \textbf{BLEU} & \textbf{R-1} & \textbf{R-2} & \textbf{R-L} & \textbf{BS-F1} &
\textbf{BLEU} & \textbf{R-1} & \textbf{R-2} & \textbf{R-L} & \textbf{BS-F1} \\
\midrule
\multicolumn{11}{c}{\textit{20\% token constraint}} \\
\midrule
Selective-Context & 6.39 & 27.33 & 7.68 & 17.93 & 84.69 & 12.16 & 46.40 & 19.31 & 28.06 & 88.32 \\
LLMLingua        & 3.36 & 26.01 & 5.53 & 13.40 & 83.22 & 8.92 & 42.15 & 14.91 & 25.29 & 87.29 \\
LongLLMLingua    & 5.19 & 35.03 & 8.61 & 17.57 & 84.51 & 6.54 & 37.86 & 11.41 & 22.26 & 86.12 \\
LLMLingua-2      & 11.82 & 46.13 & 16.62 & 25.44 & 87.96 & 13.92 & 46.95 & 20.83 & 28.63 & 88.73 \\
PartPrompt       & 10.58 & 46.09 & 16.64 & 25.13 & 87.78 & 13.95 & 46.34 & 21.24 & 28.26 & 88.45 \\
\textbf{PC-SubMax (Ours)} & \textbf{12.78} & \textbf{47.67} & \textbf{17.72} & \textbf{25.76} & \textbf{88.46} & \textbf{16.31} & \textbf{49.46} & \textbf{23.12} & \textbf{29.98} & \textbf{88.84} \\
\midrule
\multicolumn{11}{c}{\textit{30\% token constraint}} \\
\midrule
Selective-Context & 15.51 & 39.79 & 13.85 & 29.62 & 87.74 & 14.98 & 50.53 & 22.46 & 30.63 & 88.92 \\
LLMLingua        & 7.76 & 40.04 & 11.46 & 20.42 & 86.31 & 11.51 & 45.84 & 16.67 & 26.56 & 87.89 \\
LongLLMLingua    & 7.61 & 41.49 & 11.03 & 20.63 & 86.57 & 9.30 & 44.12 & 14.41 & 24.85 & 87.56 \\
LLMLingua-2      & 14.41 & 50.01 & 19.36 & 28.53 & 88.11 & 17.65 & 52.25 & 23.08 & 30.30 & 89.08 \\
PartPrompt       & 15.67 & 50.62 & 21.51 & \cs{\textbf{31.04}} & 87.72 & 18.06 & 51.48 & 23.54 & 29.32 & 88.81 \\
\textbf{PC-SubMax (Ours)} & \textbf{16.81} & \textbf{51.37} & \textbf{21.82} & 30.98 & \textbf{88.81} & \textbf{20.04} & \textbf{53.64} & \textbf{25.86} & \textbf{32.86} & \textbf{89.87} \\
\midrule
\multicolumn{11}{c}{\textit{50\% token constraint}} \\
\midrule
Selective-Context & 18.43 & 52.13 & 24.60 & 33.53 & 89.67 & 22.63 & 55.67 & 27.48 & 34.08 & 89.54 \\
LLMLingua        & 11.47 & 49.28 & 17.90 & 27.09 & 87.85 & 13.92 & 48.06 & 18.23 & 27.19 & 88.29 \\
LongLLMLingua    & 13.73 & 49.80 & 18.46 & 28.26 & 87.89 & 12.90 & 47.95 & 17.61 & 26.72 & 88.03 \\
LLMLingua-2      & 21.36 & 56.24 & 22.88 & 35.62 & 89.20 & 22.74 & 56.49 & 27.42 & 33.26 & 89.78 \\
PartPrompt       & 22.01 & 56.55 & 26.61 & 35.01 & 90.07 & 24.41 & 56.46 & 28.27 & 33.98 & 89.53 \\
\textbf{PC-SubMax (Ours)} & \textbf{25.08} & \textbf{59.64} & \textbf{30.15} & \textbf{38.47} & \textbf{90.27} & \textbf{26.83} & \textbf{58.92} & \textbf{32.64} & \textbf{38.60} & \textbf{90.58} \\
\bottomrule
\end{tabular}%
}
\endgroup
\end{table}

\subsection{Ablation Study on HotpotQA}\label{sec:app_ablation}
To examine the contribution of each semantic-utility component, we conduct an ablation study on HotpotQA by varying $\lambda_{\mathrm{cov}}$, $\lambda_{\mathrm{div}}$, and $\lambda_{\mathrm{rel}}$ while holding $\lambda_{\mathrm{tok}}=0.05$ fixed. \cs{The full configuration has the highest answer F1 at all three target ratios. Removing coverage reduces F1 by 1.88, 6.09, and 1.50 points; removing diversity reduces it by 0.64, 4.65, and 0.90 points; and removing relevance reduces it by 22.09, 31.36, and 30.89 points, respectively. The relevance-only configuration slightly exceeds the full configuration on ROUGE-1 (55.04 vs.\ 55.00) and BS-F1 (92.20 vs.\ 92.17) at CR$=0.2$. Its answer F1 remains lower at every ratio.}

\begin{table}[!htbp]
\caption{Ablation study on HotpotQA: impact of semantic-utility components. Weights are shown as \((\lambda_{\mathrm{cov}}, \lambda_{\mathrm{div}}, \lambda_{\mathrm{rel}})\), with $\lambda_{\mathrm{tok}}=0.05$ fixed.}
\label{tab:hotpotqa_ablation}
\centering
\begingroup
\footnotesize
{%
\begin{tabular}{l c rrrrrrrr}
\toprule
\textbf{Weights} & \textbf{CR} & \textbf{BLEU} & \textbf{R-1} & \textbf{R-2} & \textbf{R-L} & \textbf{BS-F1} & \textbf{P} & \textbf{R} & \textbf{F1} \\
\midrule
\multirow{3}{*}{(0.25, 0.1, 0.65)} & 0.2 & 51.18 & 55.00 & 29.27 & 54.84 & 92.17 & 59.00 & 57.41 & 56.43 \\
& 0.3 & 62.25 & 65.57 & 34.01 & 65.57 & 94.25 & 69.26 & 66.32 & 67.29 \\
& 0.5 & 67.12 & 71.82 & 39.56 & 71.82 & 94.77 & 75.18 & 74.19 & 73.24 \\
\midrule
\multirow{3}{*}{(0.0, 0.1, 0.65)} & 0.2 & 51.07 & 54.53 & 26.67 & 54.53 & 92.09 & 55.84 & 55.77 & 54.55 \\
& 0.3 & 56.01 & 60.10 & 33.38 & 60.10 & 93.20 & 62.02 & 63.23 & 61.20 \\
& 0.5 & 66.73 & 70.81 & 36.78 & 70.81 & 94.68 & 72.74 & 72.88 & 71.74 \\
\midrule
\multirow{3}{*}{(0.25, 0.0, 0.65)} & 0.2 & 51.03 & 54.40 & 26.62 & 54.18 & 92.08 & 56.64 & 55.77 & 55.79 \\
& 0.3 & 57.34 & 60.65 & 32.37 & 60.65 & 93.50 & 63.52 & 62.50 & 62.64 \\
& 0.5 & 66.96 & 70.76 & 34.67 & 70.82 & 94.72 & 74.56 & 71.27 & 72.34 \\
\midrule
\multirow{3}{*}{(0.25, 0.1, 0.0)} & 0.2 & 31.72 & 33.71 & 15.22 & 33.49 & 90.16 & 33.98 & 40.05 & 34.34 \\
& 0.3 & 34.28 & 35.33 & 16.17 & 35.11 & 90.52 & 36.23 & 38.42 & 35.93 \\
& 0.5 & 39.86 & 41.51 & 21.62 & 41.51 & 90.92 & 43.14 & 44.88 & 42.35 \\
\midrule
\multirow{3}{*}{(0.0, 0.0, 0.65)} & 0.2 & 51.01 & 55.04 & 28.26 & 53.82 & 92.20 & 58.28 & 57.40 & 55.36 \\
& 0.3 & 58.98 & 62.29 & 34.01 & 62.29 & 93.70 & 65.16 & 64.14 & 64.28 \\
& 0.5 & 66.98 & 70.57 & 38.04 & 70.61 & 94.72 & 73.93 & 72.91 & 72.07 \\
\midrule
\multirow{3}{*}{(0.0, 0.1, 0.0)} & 0.2 & 38.99 & 40.62 & 20.71 & 40.62 & 91.04 & 43.47 & 43.92 & 42.24 \\
& 0.3 & 37.88 & 41.75 & 19.35 & 41.75 & 91.58 & 45.93 & 46.87 & 43.69 \\
& 0.5 & 48.05 & 50.26 & 24.98 & 50.26 & 92.50 & 55.46 & 55.20 & 53.59 \\
\midrule
\multirow{3}{*}{(0.25, 0.0, 0.0)} & 0.2 & 27.48 & 28.88 & 10.68 & 28.88 & 88.78 & 29.26 & 31.04 & 28.95 \\
& 0.3 & 29.82 & 31.07 & 16.60 & 30.95 & 89.93 & 31.64 & 38.45 & 31.99 \\
& 0.5 & 37.34 & 38.83 & 19.28 & 38.83 & 90.73 & 40.25 & 45.15 & 40.80 \\
\bottomrule
\end{tabular}%
}
\endgroup
\end{table}



\section{Details of the Multi-Hop Relevance Construction}
\label{sec:appendix_multihop}

This appendix gives the construction of the bridge score $b_i$
used in the multi-hop relevance score $r_i^{(\mathrm{MH})}$. All quantities below are computed once, before RGM starts, and are held fixed during subset selection.

\paragraph{Conditional relevance gain.}
Given a partial evidence path $p=(i_1,\ldots,i_{t-1})$, we construct the augmented query $q_p=q\oplus s_{i_1}\oplus\cdots\oplus s_{i_{t-1}}$, where $\oplus$ denotes textual concatenation with explicit hop indicators. The conditional relevance gain of candidate $s_i$ is
\begin{equation}
g_p(i)=\left[\rho(s_i,q_p)-\rho(s_i,q)\right]_{+}.
\label{eq:conditional_relevance_gain}
\end{equation}
This incremental form avoids repeatedly rewarding information that is already directly relevant to the original query.

\paragraph{Title bridge.}
When at least $\max\{2,\lceil n/2\rceil\}$ candidates expose a \texttt{Title: sentence} structure, we additionally require an explicit title bridge. Let $T_i$ be the non-stopword terms in the title of candidate $i$, and let $E_p$ be the corresponding terms in the bodies of the preceding evidence. If $T_i=\emptyset$, we set $a_p(i)=0$; otherwise, let $\eta_p(i)=|T_i\cap E_p|/|T_i|$ and define
\begin{equation}
a_p(i)=
\begin{cases}
1, & |T_i|=1\ \text{and}\ \eta_p(i)=1,\\
\eta_p(i), & |T_i|>1\ \text{and}\ \eta_p(i)\geq1/2,\\
0, & \text{otherwise}.
\end{cases}
\label{eq:title_bridge}
\end{equation}
Parenthetical disambiguators are removed before term matching. For inputs without a prevalent title structure, we set $a_p(i)=1$. The effective transition gain is $\gamma_p(i)=g_p(i)a_p(i)$.

\paragraph{Path score and beam search.}
For an evidence path $p=(i_1,\ldots,i_h)$, we define
\begin{equation}
C(p)=\left(r_{i_1}^{(1)}\prod_{t=2}^{h}\gamma_{(i_1,\ldots,i_{t-1})}(i_t)\right)^{1/h}.
\label{eq:multihop_path_score}
\end{equation}
\cs{The geometric mean normalizes the product for path length while retaining sensitivity to weak transitions.} We expand paths with beam search of width $b=\min\{n,\max(4,2H)\}$ for a maximum of $H$ hops, and set $b_i$ to the maximum $C(p)$ among the examined paths of length $h\geq2$ that end at $s_i$, with $b_i=0$ when no such path is examined. Substituting $b_i$ into the multi-hop relevance definition yields $r_i^{(\mathrm{MH})}\in[0,1]$. \cs{On HotpotQA we use $H=2$, giving $b=\min\{n,4\}$.}

\section{Omitted Proofs for Section \ref{sc:method}}
\label{sec:appendix_proofs}
\begin{proof}[Proof of Lemma~\ref{lem:cov}]
\cs{If $W=\emptyset$, the zero function has all the stated properties. Otherwise, $f_{\mathrm{cov}}(\emptyset)=0$, so the function is normalized.} Non-negativity and monotonicity follow directly from the definition: adding sentences never decreases vocabulary coverage, and the fraction is always in $[0, 1]$. To establish submodularity, consider any $A \subseteq B \subseteq \mathcal{V}$ and $e \notin B$. The marginal gains are $f_{\text{cov}}(A \cup \{e\}) - f_{\text{cov}}(A) = \frac{|W_e \setminus \bigcup_{\cs{s_i \in A}} W_i|}{|W|}$ and $f_{\text{cov}}(B \cup \{e\}) - f_{\text{cov}}(B) = \frac{|W_e \setminus \bigcup_{\cs{s_i \in B}} W_i|}{|W|}$. Since $A \subseteq B$ implies $\bigcup_{\cs{s_i \in A}} W_i \subseteq \bigcup_{\cs{s_i \in B}} W_i$, it follows that $W_e \setminus \bigcup_{\cs{s_i \in B}} W_i \subseteq W_e \setminus \bigcup_{\cs{s_i \in A}} W_i$. Consequently, $f_{\text{cov}}(A \cup \{e\}) - f_{\text{cov}}(A) \geq f_{\text{cov}}(B \cup \{e\}) - f_{\text{cov}}(B)$, confirming submodularity.
\end{proof}

\cs{
\begin{proof}[Proof of Lemma~\ref{lem:div}]
Let $A_S=I_d+\sum_{s_i\in S}z_i z_i^\top$. Sylvester's determinant identity gives $\det(I_{|S|}+K_S)=\det A_S$. The matrix determinant lemma therefore yields
\begin{equation}
f_{\mathrm{div}}(e\mid S)=\log\bigl(1+z_e^\top A_S^{-1}z_e\bigr)\geq0.
\end{equation}
Thus $f_{\mathrm{div}}$ is monotone and, since $f_{\mathrm{div}}(\emptyset)=0$, non-negative and normalized. If $S\subseteq T$, then $A_T\succeq A_S\succ0$, so $A_T^{-1}\preceq A_S^{-1}$. The displayed marginal consequently decreases from $S$ to $T$, establishing submodularity. This argument relies on the same positive-semidefinite Gram kernel used in the objective definition.
\end{proof}
}

\begin{proof}[Proof of Lemma~\ref{lem:rel}]
The single-hop and multi-hop relevance definitions ensure $r_i\geq0$ in both modes. Although the multi-hop score is obtained from evidence paths, that computation is completed before RGM begins, so each $r_i$ is fixed during subset selection. Hence, for any $A\subseteq B\subseteq\mathcal{V}$ and $e\notin B$, $f_{\text{rel}}(A\cup\{e\})-f_{\text{rel}}(A)=r_e=f_{\text{rel}}(B\cup\{e\})-f_{\text{rel}}(B)$. The marginal is independent of the selected set, establishing modularity; non-negativity then also implies monotonicity. \cs{The empty sum is zero, establishing normalization.}
\end{proof}

\begin{proof}[Proof of Lemma~\ref{lem:bud}]
By the definition of $\ell$, $\ell(\emptyset)=0$ and $\ell(S)=\sum_{e\in S}\lambda_{\mathrm{tok}}c_e/B\geq0$. For $A\subseteq B'$, non-negative token costs imply $\ell(A)\leq\ell(B')$, establishing monotonicity. Moreover, for any $e\notin S$, $\ell(S\cup\{e\})-\ell(S)=\lambda_{\mathrm{tok}}c_e/B$, which is independent of $S$; hence $\ell$ is modular.
\end{proof}

\begin{proof}[Proof of Theorem~\ref{thm:objective}]
By Lemmas~\ref{lem:cov}--\ref{lem:rel}, each component of $U$ is normalized, non-negative, monotone, and submodular, with relevance being modular. Non-negative linear combinations preserve these properties, so $U$ has all four properties. Lemma~\ref{lem:bud} shows that $\ell$ is normalized and modular. Subtracting a modular function preserves submodularity and normalization, hence $G=U-\ell$ is normalized and submodular. Its marginal $G(e\mid S)=U(e\mid S)-\ell_e$ can be negative, so $G$ need not be monotone or non-negative on every subset.
\end{proof}

\cs{
\begin{proof}[Proof of Theorem~\ref{thm:pc_guarantee}]
Discard candidates with $c_e>B$ and normalize the capacity to one by
writing $w_e=c_e/B$.  Keep the modular prices $\ell_e$ unchanged and
define $H=U-2\ell$.  Thus, $H$ is normalized and submodular, although
it need not be monotone or non-negative on arbitrary subsets.  RGM
selects a feasible item maximizing $H(e\mid S)/w_e$ and commits it only
when its $H$-marginal is strictly positive.

Fix any feasible comparator $O$ and let $R=H(O)=U(O)-2\ell(O)$.
The empty incumbent guarantees $G(Q)\geq0$, so there is nothing to prove
when $R\leq0$.  Assume $R>0$.

First suppose RGM terminates before any unselected item of $O$ becomes
infeasible.  At the terminal prefix $S$, every item of $O\setminus S$
fits and has $H(o\mid S)\leq0$; this is vacuous if $O\subseteq S$.
Monotonicity and submodularity of $U$ imply
\begin{align}
U(O)
&\leq U(S\cup O)
 \leq U(S)+\sum_{o\in O\setminus S}U(o\mid S)\\
&\leq U(S)+2\ell(O\setminus S).
\end{align}
Consequently,
\begin{equation}
R\leq U(S)-2\ell(O\cap S)\leq U(S).
\end{equation}
Every committed marginal of $H$ is positive, and therefore $H(S)\geq0$.
Since $G(S)=\tfrac12 U(S)+\tfrac12 H(S)$, the recorded prefix $S$
satisfies $G(Q)\geq G(S)\geq R/2$.

Otherwise, let $S\to S\cup\{a\}$ be the first committed step after
which an unselected item of $O$ becomes infeasible.  Before this step,
all items of $O\setminus S$ fit individually.  Choose a maximum-weight
item $o_1\in O\setminus S$, and write
\begin{equation}
x=w(S),\qquad w_1=w_{o_1},\qquad
d=\frac{H(a\mid S)}{w_a}>0,\qquad
q_1=G(S\cup\{o_1\}).
\end{equation}
The augmentation rule tests a set with value at least $q_1$, so
$G(Q)\geq q_1$.

The committed $H$-densities are non-increasing along the greedy path.
Indeed, an item selected later was feasible at every earlier prefix;
submodularity can only decrease its marginal, and each earlier
selection maximized the feasible marginal density.  It follows that
\begin{equation}
H(S)\geq x d.
\label{eq:rgm_prefix_certificate}
\end{equation}
This bound follows from submodularity of $H$ and the density-selection rule.

Set $T=O\setminus(S\cup\{o_1\})$.  Expanding the modular costs gives
\begin{align}
R
&\leq U(S\cup O)-2\ell(O)\\
&\leq U(S\cup\{o_1\})
 +\sum_{e\in T}U(e\mid S\cup\{o_1\})-2\ell(O)\\
&=q_1+\ell(S)-\ell_{o_1}-2\ell(O\cap S)
 +\sum_{e\in T}H(e\mid S\cup\{o_1\})\\
&\leq q_1+\ell(S)+\sum_{e\in T}H(e\mid S)\\
&\leq q_1+\ell(S)+w(T)d\\
&\leq q_1+\ell(S)+(1-w_1)d.
\label{eq:regularized_greedymax}
\end{align}
Here the fourth line uses non-negativity of the modular prices and
submodularity of $H$.  The fifth line uses the greedy rule, because
all elements of $T$ are feasible at $S$.  The last line follows from
$w(O)\leq1$ and $d>0$.

By the definition of the crossing step, some
$o\in O\setminus(S\cup\{a\})$ satisfies $x+w_a+w_o>1$.
Because $w_o\leq w_1$,
\begin{equation}
x+w_a>1-w_o\geq1-w_1.
\end{equation}
Using~\eqref{eq:rgm_prefix_certificate} and $\ell_a\geq0$, the next
recorded prefix satisfies
\begin{align}
G(S\cup\{a\})
&=H(S)+w_a d+\ell(S)+\ell_a\\
&\geq(x+w_a)d+\ell(S)\\
&\geq(1-w_1)d+\ell(S).
\end{align}
Combining this with~\eqref{eq:regularized_greedymax} yields
\begin{equation}
R\leq q_1+G(S\cup\{a\})\leq2G(Q),
\end{equation}
since RGM records both the augmentation and the next greedy prefix.
Thus $G(Q)\geq R/2=\tfrac12U(O)-\ell(O)$ in every case.

Finally, at most $\kappa$ items can be committed.  At each prefix with
at least one feasible candidate, RGM evaluates at most $n$ marginals
and reuses the same marginal in both selection rules.  A prefix with
$\kappa$ items has no feasible augmentation by the definition of
$\kappa$.  Hence there are at most $\kappa$ such scans and
$O(n\kappa)$ value-oracle queries.
\end{proof}
}

\cs{
\section{A lazy accelerated implementation of RGM}
\label{sec:fast_rgm}

The full-scan solver in Algorithm~\ref{alg:pc_submax} can be accelerated by replacing its two exhaustive scans with approximate lazy selectors. Marginal evaluations themselves remain exact; $\varepsilon$ controls the accuracy of the selection rules. This section establishes the resulting optimization guarantee; the experimental results in the paper concern the full-scan solver. Keep $n=|\mathcal V|$ and discard candidates with $c_e>B$. If none remain, return $\emptyset$ without querying $U$. Otherwise assume $c_e>0$ and write
\[
G(S)=U(S)-\ell(S),\qquad H(S)=U(S)-2\ell(S),\qquad
E(S)=\{e\in\mathcal V\setminus S:c(S)+c_e\le B\}.
\]
The proof needs only that $U$ is normalized, non-negative, monotone and submodular, and that $\ell$ is non-negative modular; in particular, it covers the aligned token prices used in PC-SubMax.

\paragraph{Selector requirements.}
For $0<\varepsilon<1/2$, put
\[
\Lambda=\max_{e:c_e\le B}[G(\{e\})]_+,
\qquad q=1-\varepsilon,\qquad
\delta=\varepsilon\Lambda,\qquad \tau=\frac{\varepsilon\Lambda}{B},
\]
where $[x]_+=\max\{x,0\}$. If $\Lambda=0$, return the empty set. Otherwise, retain the best singleton as an initial incumbent. We write $S+e=S\cup\{e\}$ and $S+\bot=S$. If $E(S)=\emptyset$, both selectors return $\bot$ and the maxima of feasible positive marginals are taken to be zero. At each prefix $S$, the augmentation selector returns $b\in E(S)\cup\{\bot\}$ satisfying
\begin{equation}
G(S+b)\ge G(S)+\max\{0,\max_{e\in E(S)}G(e\mid S)\}-\delta.
\label{eq:fast_augmentation_condition}
\end{equation}
The density selector either returns $a\in E(S)$ satisfying
\begin{equation}
\frac{H(a\mid S)}{c_a}\ge q\max_{e\in E(S)}\frac{[H(e\mid S)]_+}{c_e},
\qquad \frac{H(a\mid S)}{c_a}\ge\tau,
\label{eq:fast_density_condition}
\end{equation}
or returns $\bot$ and certifies that all feasible distorted densities are below $\tau$. As in full-scan RGM, compare $S$ and $S+b$ with the incumbent, stop if the density selector returns $\bot$, and otherwise commit $a$ to the path. Augmentation does not change the path.

\paragraph{Persistent augmentation queue.}
Initialize a maximum-priority queue with keys $u_e=[G(\{e\})]_+$. At each call, repeatedly remove selected or currently infeasible elements whenever they reach the top. If the queue is empty, return $\bot$. Otherwise pop its largest key $u_e$, evaluate $v_e=[G(e\mid S)]_+$, and let $u_{(2)}$ be the largest remaining key, or zero if none remains. If $v_e\ge u_{(2)}-\delta$, reinsert $e$ with updated key $v_e$ and return $e$ when $v_e>0$, or return $\bot$ when $v_e=0$. Otherwise reinsert it with key $v_e$ and repeat. Reinsertion on success is essential: this call only records an augmented candidate, so $e$ may be needed at a later prefix. A zero-key element can instead be discarded permanently from this queue.

\paragraph{Persistent density queue.}
Maintain a separate maximum-priority queue. Initialize the key for each element to $d_e=[H(\{e\})]_+/c_e$. Repeatedly remove selected or currently infeasible elements whenever they reach the top. If the queue is empty or its largest key is below $\tau$, return $\bot$. Otherwise pop its largest-key element $e$ and evaluate $v_e=[H(e\mid S)]_+/c_e$. If $v_e<\tau$, discard $e$ from this queue and repeat. If $v_e\ge\tau$, let $d_{(2)}$ be the largest remaining key, or zero if none remains. Return $e$ if $v_e\ge qd_{(2)}$; it is then immediately committed to $S$ and removed from both queues. Otherwise reinsert it with key $v_e$ and repeat. An element discarded because its distorted density is below $\tau$ must remain eligible for augmentation: $G(e\mid S)$ can still be positive. Only selection or infeasibility permits deletion from both queues. Break all ties by the original sentence index.

\begin{theorem}[Accelerated RGM]
\label{thm:fast_rgm}
For every $\varepsilon\in(0,1/2)$, the above implementation deterministically returns a feasible set $Q$ such that, for every feasible comparator $O$,
\begin{equation}
G(Q)\ge\left(\frac12-\varepsilon\right)U(O)-\ell(O).
\label{eq:fast_rgm_guarantee}
\end{equation}
When at least one item is individually feasible, let $c_{\min}=\min_{e:c_e\le B}c_e$. The algorithm uses
\begin{equation}
O\!\left(\frac{n}{\varepsilon}
\left[1+\log\!\left(\frac{B}{\varepsilon c_{\min}}\right)\right]\right)
\label{eq:fast_rgm_queries}
\end{equation}
value-oracle queries and $O(n)$ additional storage for the selectors and the greedy path. This storage bound excludes the embedding/kernel representation and the state needed to evaluate utility marginals.
\end{theorem}

\paragraph{Validity and query complexity of the selectors.}
Both $G$ and $H$ are submodular. Their positive-part marginals therefore do not increase as the prefix grows, so all cached keys remain upper bounds. On a successful augmentation refresh, every other current positive marginal is at most $u_{(2)}\le v_e+\delta$, proving~\eqref{eq:fast_augmentation_condition}. On a failed refresh, $v_e<u_{(2)}-\delta\le u_e-\delta$. Keys start at most $\Lambda$ and remain non-negative, so each element incurs at most $O(1/\varepsilon)$ failed augmentation refreshes. There is at most one successful refresh per prefix and at most $n+1$ prefixes, yielding $O(n/\varepsilon+n)$ augmentation queries in total. Successful candidates remain in the queue for later calls.

A returned density candidate has true density at least $qd_{(2)}$ and $\tau$, proving~\eqref{eq:fast_density_condition}. Every permanently discarded density key has true density below $\tau$ forever, which also proves the stopping certificate. On a failed density refresh, $v_e<qd_{(2)}\le qd_e$. Moreover,
\[
d_e^{(0)}=\frac{[H(\{e\})]_+}{c_e}\le\frac{\Lambda}{c_e}.
\]
Before an element drops below the floor, it can thus incur only
\[
O\!\left(\frac1\varepsilon\log\frac{B}{\varepsilon c_e}\right)
\]
failed density refreshes. Each element is successfully committed or permanently discarded at most once. Initialization and incumbent-value bookkeeping require only $O(n)$ further queries, giving~\eqref{eq:fast_rgm_queries}. Standard binary heaps add a factor $O(\log n)$ to the queue-operation count; this is separate from the number and cost of oracle evaluations.

\paragraph{Approximation analysis.}
Write $F=U(O)$, $L=\ell(O)$ and
\[
T=\left(\frac12-\varepsilon\right)F-L,\qquad \beta=q^2.
\]
The incumbent includes the empty set, so $T\le0$ is immediate. Assume $T>0$, which implies that $O$ is nonempty. If $\Lambda>F$, the retained best singleton already gives $G(Q)\ge\Lambda>F\ge T$. It therefore remains to consider
\begin{equation}
\Lambda\le F.
\label{eq:fast_singleton_bound}
\end{equation}
If $\Lambda=0$, normalized submodularity gives $G(A)\le\sum_{e\in A}G(\{e\})\le0$ for every $A$, contradicting $T>0$. Hence $\Lambda>0$. Every committed item has positive distorted marginal, so every path prefix $S$ satisfies $H(S)\ge0$.

Let $C=c(O)$, choose a heaviest comparator item $o^\star\in O$, and write $r=c_{o^\star}$. At any prefix of cost at most $C-r$, every unselected item of $O$ is feasible. If the path terminates at such a prefix $S$, the stopping certificate and monotone submodularity imply
\begin{align*}
F-U(S)
&\le\sum_{o\in O\setminus S}U(o\mid S)\\
&=\sum_{o\in O\setminus S}\bigl(H(o\mid S)+2\ell_o\bigr)
\le\tau C+2L\le\varepsilon\Lambda+2L.
\end{align*}
Consequently,
\[
G(S)=\tfrac12[U(S)+H(S)]\ge\tfrac12F-L-\tfrac12\varepsilon\Lambda\ge T.
\]

Otherwise let $P$ be the last prefix with $x=c(P)\le C-r$, and let $a$ be the next committed item, so $x+c_a>C-r$. Put
\[
d=H(a\mid P)/c_a,\qquad D=d/q,\qquad
L_{\mathrm{in}}=\ell(P\cap O),\qquad L_{\mathrm{out}}=\ell(P\setminus O).
\]
Because every unselected comparator item fits at $P$,~\eqref{eq:fast_density_condition} gives
\begin{equation}
H(o\mid P)\le Dc_o\quad(o\in O\setminus P).
\label{eq:fast_opt_density}
\end{equation}
The future item $a$ was feasible at every earlier prefix, where its distorted density was at least $d$ by submodularity. Thus every earlier selected density was at least $qd$, yielding $H(P)\ge qdx$. Since the last selected density is $d\ge qd$, we obtain
\begin{equation}
H(P+a)\ge \beta D(x+c_a)>\beta D(C-r).
\label{eq:fast_prefix_certificate}
\end{equation}
We also need the scalar estimate
\begin{equation}
\beta[H(O)-T-\delta]\ge T.
\label{eq:fast_scalar}
\end{equation}
Indeed, ~\eqref{eq:fast_singleton_bound} implies $H(O)-T-\delta\ge F/2-L$, and hence the left-hand side minus $T$ is at least
\[
\left(\frac\beta2-\frac12+\varepsilon\right)F+(1-\beta)L
=\frac{\varepsilon^2}{2}F+(1-\beta)L\ge0.
\]

If $o^\star\notin P$, ~\eqref{eq:fast_opt_density}, submodularity of $H$, and monotonicity of $U$ imply
\begin{align*}
H(o^\star\mid P)+(C-r)D
&\ge\sum_{o\in O\setminus P}H(o\mid P)\\
&\ge H(O\setminus P\mid P)\\
&\ge F-U(P)-2(L-L_{\mathrm{in}}).
\end{align*}
Rearranging gives
\[
G(P+o^\star)+(C-r)D
\ge H(O)+\ell_{o^\star}+L_{\mathrm{in}}-L_{\mathrm{out}}.
\]
Let $A_P$ be the recorded augmentation at $P$.~\eqref{eq:fast_augmentation_condition} gives $G(A_P)\ge G(P+o^\star)-\delta$. If $G(A_P)\ge T$, the incumbent suffices. Otherwise,
\[
(C-r)D>H(O)-T-\delta+\ell_{o^\star}+L_{\mathrm{in}}-L_{\mathrm{out}}.
\]
Together with~\eqref{eq:fast_prefix_certificate}--\eqref{eq:fast_scalar}, this yields
\begin{align*}
G(P+a)
&=H(P+a)+L_{\mathrm{in}}+L_{\mathrm{out}}+\ell_a\\
&>\beta[H(O)-T-\delta]
+\beta\ell_{o^\star}+(1+\beta)L_{\mathrm{in}}\\
&\quad +(1-\beta)L_{\mathrm{out}}+\ell_a\ge T.
\end{align*}
If $o^\star\in P$, then $c(O\setminus P)\le C-r$, and the same argument without the separate $o^\star$ term gives
\[
G(P)+(C-r)D\ge H(O)+L_{\mathrm{in}}-L_{\mathrm{out}}.
\]
Either $G(P)\ge T$, or~\eqref{eq:fast_prefix_certificate}--\eqref{eq:fast_scalar} again give $G(P+a)>T$. Finally, when $C=r$, positive costs make $O$ a singleton and the initial best-singleton incumbent already achieves $G(O)\ge T$. All sets used in the argument are recorded by the algorithm and remain feasible, completing the proof.

\paragraph{Interpretation for PC-SubMax.}
The accelerated bound replaces the $\kappa$ factor by $\varepsilon^{-1}[1+\log(B/(\varepsilon c_{\min}))]$, trading $\varepsilon$ in the utility coefficient for reduced query complexity when this factor is smaller. Its total implementation cost also includes preprocessing, marginal evaluation, and priority-queue operations. The reported experiments use full-scan RGM.
}

\section{\cs{Detailed Computational Complexity Analysis}}
\label{sec:appendix_complexity}

We provide a comprehensive breakdown of the computational complexity
of Algorithm~\ref{alg:pc_submax}. Let $n$ denote the number of
candidate sentences, $\kappa$ the maximum number of sentences that can
fit within the token budget $B$, $d$ the embedding dimension, and
$\bar{L}$ the average number of unique words per sentence.

\subsection{Preprocessing Stage}

\cs{Let $T_{\mathrm{text}}$ denote the cost of sentence segmentation, token-cost computation, vocabulary construction, and any title-term preprocessing. This includes scanning word occurrences in the input, whose number can exceed the number of unique words. Let $T_{\mathrm{enc,all}}$ denote the total cost of encoding and normalizing the candidate sentences and any base query. Encodings may be processed in batches; the operation count is distinct from measured GPU latency. Under constant-time dictionary operations, the resulting vocabulary sets occupy $O(n\bar L)$ storage.}

\paragraph{Relevance computation.}
\cs{Single-hop similarities require $O(nd)$ operations once normalized embeddings are available. For multi-hop relevance, at most $(H-1)b$ retained path prefixes produce augmented queries. Let $T_q^{(H)}$ bound the cost of constructing and encoding one augmented query and forming the corresponding prefix-term set. Computing similarities and title matches for that prefix costs $O(nd+n\bar L)$. Maintaining the best $b$ extensions per hop can be implemented in $O(bn\log(1+b))$ operations. The additional multi-hop cost is therefore
\begin{equation}
T_{\mathrm{MH}}=
O\!\left((H-1)b(T_q^{(H)}+nd+n\bar L)
          +Hbn\log(1+b)\right).
\end{equation}
All relevance scores are stored before subset selection. In single-hop mode, $T_{\mathrm{MH}}=0$. With $H=2$ and $b=\min\{n,4\}$, the hop limit and beam width are bounded.}

\paragraph{Dense kernel construction.}
\cs{Let $Z\in\mathbb R^{n\times d}$ contain the unit-normalized sentence embeddings. Computing $K=ZZ^\top$ costs $O(n^2d)$ operations and requires $O(n^2)$ storage. Following the main text, $T_{\mathrm{prep}}$ excludes this matrix multiplication and satisfies
\begin{equation}
T_{\mathrm{prep}}=
O(T_{\mathrm{text}}+T_{\mathrm{enc,all}}+nd+T_{\mathrm{MH}}).
\end{equation}
The total preprocessing cost including the dense kernel is
$O(T_{\mathrm{prep}}+n^2d)$.}

\subsection{Selection Stage: Per-Query Oracle Cost}
At each RGM prefix, we evaluate $U(e\mid S)$ once for every feasible candidate $e$ and reuse it in both the augmentation and cost-scaled density criteria. By the definition of $U$, this marginal decomposes into coverage, diversity, and relevance. The token-price marginal $\ell_e$ is then subtracted once for the augmentation score or twice for the path score.

\paragraph{(a) Coverage Marginal Gain: \cs{$O(|W_e|)$ per candidate.}}
We maintain a global hash set $C = \bigcup_{\cs{s_i \in S}} W_i$ that 
tracks the vocabulary currently covered by the selected subset $S$. 
For each candidate $e$, the marginal gain is:
\begin{equation}
\Delta_e f_{\text{cov}}(S) 
= \frac{|W_e \setminus C|}{|W|},
\end{equation}
which requires iterating over the words in $W_e$ and checking 
membership in $C$. \cs{Under constant-time dictionary membership operations,} 
\cs{this costs $O(|W_e|)$ for candidate $e$, and updating the union after accepting it also costs $O(|W_e|)$. A complete scan costs $O(\sum_e |W_e|)=O(n\bar L)$, and at most $\kappa$ nonempty scans give $O(n\kappa\bar L)$ for $\kappa\geq1$. If $W=\emptyset$, all coverage marginals are zero.}

\paragraph{(b) Diversity Marginal Gain: $O(j^2)$ per candidate via 
Incremental Cholesky.}
\cs{The cost of this term grows with the number of selected sentences.} At iteration $j$ (where $|S| = j$), we maintain the 
Cholesky factorization $K_S + I = LL^\top$, where 
$L \in \mathbb{R}^{j \times j}$ is lower triangular. For a candidate 
$e$, the augmented matrix can be written in block form:
\begin{equation}
K_{S \cup \{e\}} + I 
= \begin{pmatrix} 
K_S + I & \mathbf{k}_e \\ 
\mathbf{k}_e^\top & K_{ee} + 1 
\end{pmatrix},
\end{equation}
where $\mathbf{k}_e = (K_{e,s})_{s \in S} \in \mathbb{R}^{j}$ is 
the column of similarities between $e$ and the elements currently in 
$S$. By the Schur complement formula:
\begin{equation}
\det(K_{S \cup \{e\}} + I) 
= \det(K_S + I) \cdot \sigma_e^2,
\end{equation}
where the Schur complement is:
\begin{equation}
\sigma_e^2 
= K_{ee} + 1 
- \mathbf{k}_e^\top (K_S + I)^{-1} \mathbf{k}_e.
\end{equation}
Taking logarithms on both sides yields:
\begin{equation}
\Delta_e f_{\text{div}}(S) 
= \log\det(K_{S \cup \{e\}} + I) - \log\det(K_S + I) 
= \log \sigma_e^2.
\label{eq:marginal_gain}
\end{equation}

To compute $\sigma_e^2$ efficiently, we exploit the existing Cholesky 
factor $L$. Solving the triangular system $L\mathbf{z} = \mathbf{k}_e$ 
via forward substitution costs $O(j^2)$, after which:
\begin{equation}
\sigma_e^2 = K_{ee} + 1 - \|\mathbf{z}\|^2,
\label{eq:sigma_e} 
\end{equation}
where the inner product $\|\mathbf{z}\|^2$ is computed in $O(j)$. 
\cs{Thus a diversity marginal costs $O(1+j^2)$, including the empty-prefix case, compared with $O((j+1)^3)$ for recomputing the augmented determinant.}

When a sentence $e$ is actually added to $S$, the Cholesky factor is 
updated by appending a new row:
\begin{equation}
L' = \begin{pmatrix} 
L & \mathbf{0} \\ 
\mathbf{z}^\top & \sigma_e 
\end{pmatrix},
\end{equation}
where $\mathbf{z}$ is already computed from the evaluation step and 
$\sigma_e = \sqrt{\sigma_e^2}$. \cs{With preallocated factor storage and the selected candidate's triangular-solve result cached, appending the row costs $O(j+1)$.}

\paragraph{(c) Relevance Marginal Gain: $O(1)$ per candidate.}
Since $f_{\text{rel}}$ is modular, the marginal gain is simply the 
precomputed relevance score:
\begin{equation}
\Delta_e f_{\text{rel}}(S) = r_e,
\end{equation}
which is retrieved from the lookup table in $O(1)$ time.

\paragraph{(d) Token-Price Marginal: $O(1)$ per candidate.}
The token price $\ell$ is modular, and its marginal is
\begin{equation}
\ell(e\mid S)=\ell_e=\lambda_{\mathrm{tok}}\frac{c_e}{B}.
\end{equation}
The value is precomputed for each sentence, so both $U(e\mid S)-\ell_e$
and $(U(e\mid S)-2\ell_e)/c_e$ add only $O(1)$ arithmetic after the
shared utility-marginal evaluation.

\paragraph{Total Per-Query Cost.}

Combining all components, a single exact marginal evaluation at
prefix size $j$ costs
\begin{equation}
\cs{O(|W_e|)+O(1+j^2)+O(1)=O(1+|W_e|+j^2).}
\end{equation}
\cs{The triangular-solve cost grows quadratically with $j$, while the coverage cost depends on $|W_e|$; relevance and token price are constant-time lookups.}

\subsection{Total Selection Cost}

\cs{If no item is individually feasible, RGM returns the empty set after $O(n)$ cost checks. Otherwise $\kappa\geq1$. Every nonempty candidate scan occurs at a prefix size $j\leq\kappa-1$, since a feasible set of size $\kappa$ cannot be augmented. Let $E_j$ be the feasible candidates at prefix size $j$, taking $E_j=\emptyset$ if that prefix is not reached. The selection cost is bounded by
\begin{align}
\sum_{j=0}^{\kappa-1}\sum_{e\in E_j}O(1+|W_e|+j^2)
&\leq \sum_{j=0}^{\kappa-1}O(n+n\bar L+nj^2)\\
&=O(n\kappa\bar L+n\kappa^3).
\end{align}
Coverage-union updates, factor updates, feasibility checks, and incumbent bookkeeping fit within this bound. Both selection criteria reuse the same utility marginal for a candidate.}

\subsection{Overall Complexity}

\cs{Adding preprocessing gives the runtime stated in Theorem~\ref{thm:complexity}:
\begin{equation}
O(T_{\mathrm{prep}}+n^2d+n\kappa\bar L+n\kappa^3).
\end{equation}
With bounded sentence lengths and fixed encoder configuration, encoder input limits, hop limit, and beam width, this simplifies to $O(n^2d+n\kappa^3)$.}

\paragraph{Scaling and memory.}
\cs{The parameter $\kappa$ is the maximum feasible cardinality. With unequal sentence costs, $\kappa/n$ can differ from the requested token-retention ratio. If $\kappa=\Theta(n)$, the dense triangular-solve implementation has a selection upper bound of $O(n^4)$. Its auxiliary storage for the kernel, embeddings, vocabulary sets, and Cholesky factor is $O(n^2+nd+n\bar L+\kappa^2)$, excluding the raw text and encoder state.}

\cs{An alternative exact representation follows from Sylvester's determinant identity:
\begin{equation}
\det(I_{|S|}+Z_SZ_S^\top)=\det(I_d+Z_S^\top Z_S).
\end{equation}
Maintaining the $d\times d$ matrix on the right permits exact diversity-marginal evaluation in embedding space and can avoid storing the dense $n\times n$ kernel. This offers a complementary implementation option when the candidate set is large relative to $d$. Approximate kernels offer another direction for scaling, with optimization guarantees stated for the corresponding approximate utility.}

\end{document}